\documentclass[10pt]{article}

\usepackage[margin=1.1in]{geometry}
\usepackage{amsmath,amssymb,amsfonts,amsthm}
\usepackage{mathtools}
\usepackage{array}
\usepackage{graphicx}
\usepackage{xcolor}
\usepackage{float}
\usepackage[authoryear,round]{natbib}
\AtBeginDocument{%
  \let\ExTRAoriginalbibitem\bibitem
  \renewcommand{\bibitem}{\normalfont\ExTRAoriginalbibitem}}
\usepackage{hyperref}
\usepackage{cleveref}
\usepackage{enumitem}
\usepackage{bm}
\usepackage{titlesec}
\usepackage{fancyhdr}
\usepackage{microtype}
\usepackage{parskip}
\usepackage{placeins}
\usepackage{flafter}

\usepackage{tikz}
\usepackage{pgfplots}
\pgfplotsset{compat=1.18}
\usetikzlibrary{arrows.meta,positioning,calc}
\definecolor{cpgray}{RGB}{95,103,112}
\definecolor{wcpblue}{RGB}{25,105,157}
\definecolor{tiltorange}{RGB}{202,90,36}
\pgfplotsset{
  resultaxis/.style={
    width=0.47\linewidth,height=4.6cm,
    scale only axis=false,
    axis lines=left,grid=major,grid style={black!8},
    tick label style={font=\scriptsize},
    label style={font=\small},title style={font=\small},
    legend style={font=\scriptsize,draw=none,fill=none},
    tick align=outside,major tick length=2pt,
    scaled ticks=false,
    every axis plot/.append style={line width=0.9pt,mark size=2pt}},
  cpplot/.style={cpgray,mark=square*},
  wcpplot/.style={wcpblue,mark=*},
  tiltplot/.style={tiltorange,mark=triangle*}
}

\newtheorem{proposition}{Proposition}

\hypersetup{
  colorlinks=true,
  linkcolor=blue!70!black,
  citecolor=blue!50!black,
  urlcolor=blue!60!black,
  pdftitle={Conformal Prediction under Exponential-Tilt Joint Shift},
  pdfauthor={Seungjin Choi}
}

\providecommand{\E}{\mathbb{E}}

\newcommand{\R}{\mathbb R}
\newcommand{\Pp}{\mathbb P}
\newcommand{\ind}{\mathbf 1}
\newcommand{\Dtr}{\mathcal D_{\mathrm{tr}}}
\newcommand{\Dshift}{\mathcal D_{\mathrm{shift}}}
\newcommand{\Dcal}{\mathcal D_{\mathrm{cal}}}
\newcommand{\C}{\mathcal C}
\newcommand{\X}{\mathcal X}
\newcommand{\Y}{\mathcal Y}
\newcommand{\TV}{d_{\mathrm{TV}}}

\titleformat{\section}{\large\bfseries\color{blue!60!black}}{\thesection}{1em}{}[\titlerule]
\titleformat{\subsection}{\normalsize\bfseries\color{blue!40!black}}{\thesubsection}{1em}{}
\titleformat{\subsubsection}{\normalsize\itshape\bfseries}{\thesubsubsection}{1em}{}

\begin{document}

\begin{center}
  {\LARGE\bfseries \textsf{Conformal Prediction under Exponential-Tilt Joint Shift}}\\[2em]
  {\large Seungjin Choi}\\[1em]
  {\normalsize CROID Research and aSSIST University, Seoul, Korea}
  \end{center}

\vspace{0.5em}
\noindent\rule{\linewidth}{1.5pt}
\vspace{0.5em}

\begin{abstract}
Conformal prediction can lose coverage when the data distribution changes
after deployment. We study adaptation using labeled source data and
unlabeled target inputs, allowing both the input distribution and its
relationship with outcomes to change. We use Exponential Tilt
Reweighting Alignment (ExTRA), introduced for classification by
\citet{MaityS2023iclr}, to estimate structured distribution shifts.
We compare using its estimated weights in conformal calibration with
additionally tilting the source predictive distribution. Shared learned
predictors, estimated weights, calibration samples, and test observations
isolate the effect of tilting. Existing theory gives both procedures
target coverage with true weights and a common coverage bound with
estimated weights. Identification calculations and an analysis of how
scoring interacts with weight estimation error help explain why their
performance can nevertheless differ.
In a synthetic regression setting where the assumed models match
the data-generating process and target inputs are informative about
the shift, tilting reduces mean set length by about $30\%$ relative
to weighting alone, with both methods attaining coverage near nominal.
Tilting can instead cause substantial coverage losses in synthetic
classification and in regression when target inputs provide little
information about the response shift. Real-data experiments also show no consistent benefit.
Good coverage from weighted calibration alone does not ensure that
adding predictive tilting will preserve coverage. Deciding when to
apply this additional adjustment using only source labels and target
inputs remains an open problem.
\end{abstract}

\section{Introduction}
\label{sec:introduction}

Distribution changes between training and deployment can undermine
predictive accuracy and the reliability of uncertainty estimates
\citep{KohPW2021icml,OvadiaY2019neurips}. For prediction sets, the
concern is whether they continue to contain the true outcome at the
intended rate. We study this problem using labeled source data and
unlabeled target inputs. These observations alone cannot identify an
arbitrary target joint distribution, so adaptation requires assumptions
linking the two populations \citep{Ben-DavidS2010aistats}.

Let $P$ and $Q$ denote the source and target distributions of $(X,Y)$.
Standard split conformal prediction provides marginal coverage under
exchangeability \citep{VovkV2005book,LeiJ2018jasa,AngelopoulosAN2023ftml}.
When $Q\ll P$, weighted conformal prediction extends this guarantee
using the true target-to-source density ratio $dQ/dP$
\citep{TibshiraniR2019neurips}. This ratio may depend on both the input
and the response. The practical difficulty is estimating it without
observing target responses.

Two common shift assumptions simplify this estimation problem.
Under \emph{covariate shift}, $Q_{Y\mid X}=P_{Y\mid X}$, so the joint
ratio reduces to the input ratio $dQ_X/dP_X$. This is the basis of
classical importance weighting
\citep{ShimodairaH2000jspi,SugiyamaM2007jmlr_a}.
Under \emph{label shift}, $Q_{X\mid Y}=P_{X\mid Y}$, so the ratio
depends only on the response. For classification, target class
proportions can be estimated under suitable identification conditions
through likelihood adjustment
\citep{SaerensM2002neco,AlexandariAM2020icml}
or confusion-matrix matching \citep{LiptonZ2018icml}.
Regularization and source probability calibration also affect estimation
accuracy \citep{AzizzadenesheliK2019iclr,GargS2020neurips}.
\citet{PodkopaevA2021uai} study conformal prediction under label shift.
Our earlier work applies weighted conformal prediction under label
shift to molecular property prediction
\citep{LeeHS2025reliableML}.
For continuous responses, the same invariance is also called
\emph{target shift} \citep{YiM2025arxiv}.

More general joint shifts require other structural restrictions
\citep{ZhangK2013icml}. Exponential Tilt Reweighting Alignment (ExTRA),
introduced for classification by \citet{MaityS2023iclr}, models the joint
ratio through exponential tilts, building on classical density ratio
models
\citep{AndersonJ1979biometrika,QinJ1998biometrika,FokianosK2001technometrics}.
Each proposed tilt implies an input distribution after averaging over
the source conditional response distribution. ExTRA estimates the tilt
by matching this induced input distribution to the observed target
inputs. For regression, conditional integration replaces the sum over
classes. Whether target inputs identify the joint ratio depends on the
chosen features and the source distribution.

We compare two ways of using the fitted ratio. ExTRA-WCP uses it in
weighted calibration while retaining the source score.
ExTRA-WCP-T uses the same weighted calibration but also tilts the source
predictive distribution before computing the score. The additional
adjustment changes which responses the predictor regards as plausible
and therefore changes the shape of the prediction set.

This predictive adjustment has precedents. Under label shift, it
multiplies the source class probabilities by target-to-source class-prior
ratios and renormalizes them \citep{SaerensM2002neco}.
Predictive tilting under label shift also appears in our conformal Bayes work
\citep{Choi2026eiml,Choi2026arxiv_scbc} and in a recent application
to molecular property prediction with full conformal Bayes \citep{KimJY2026copa}.
CPUTS combines predictive adjustment with weighted calibration under
regression target shift \citep{YiM2025arxiv}. Here we consider ratios
that may depend jointly on $(X,Y)$ and source predictors that need
not be Bayesian.

Weighting and tilting affect different parts of the construction.
Weighting changes how calibration observations and the candidate
response enter the conformal rank. Tilting changes their scores.
A predictive distribution better aligned with the target can yield
smaller sets at the required coverage. With an estimated ratio, however,
changing the score can also move the set's misses toward regions where
the ratio underestimates target probability. The two scores share a
coverage bound determined by ratio error, but their actual coverage
can differ substantially.

Our contribution is a controlled comparison that holds the learned
source predictor, fitted ratio, calibration sample, and test observations
fixed to isolate the effect of tilting. Existing coverage theory and
an analysis of score-dependent errors support the comparison.
Synthetic regression demonstrates substantial length gains with
coverage near nominal, while synthetic classification, regression with
weak information about the response shift, and real-data experiments
reveal coverage losses. A two-stage baseline for separable shifts and
an experiment with an input-dependent response shift help distinguish
ratio estimation error from restrictions of the shift model.
Identification calculations explain which shifts target inputs can
distinguish within the experimental families.

\section{ExTRA and identification}
\label{sec:extra}

\subsection{Data and shift model}
\label{subsec:tilt-model}

The four mutually independent samples, each containing independent
observations, are
\begin{align*}
 \Dtr&=\{(X_i^{\mathrm{tr}},Y_i^{\mathrm{tr}})\}_{i=1}^{n_{\mathrm{tr}}}
       \sim P^{n_{\mathrm{tr}}}, &
 \Dshift^P&=\{(X_i^P,Y_i^P)\}_{i=1}^{n_{\mathrm{shift}}}
       \sim P^{n_{\mathrm{shift}}},\\
 \Dshift^Q&=\{X_j^Q\}_{j=1}^m\sim Q_X^m, &
 \Dcal&=\{(X_i,Y_i)\}_{i=1}^n\sim P^n.
\end{align*}
We use $\Dtr$ to fit the source conditional model and
$(\Dshift^P,\Dshift^Q)$ to estimate the joint density ratio.
The sample $\Dcal$ is reserved for conformal calibration.
All fitted models, score functions, weights, and tuning choices
are fixed before calibration and testing. For an independent
target pair $(X_{n+1},Y_{n+1})\sim Q$, we observe only $X_{n+1}$
and seek a prediction set $\C(X_{n+1})$ that contains $Y_{n+1}$
with marginal probability at least $1-\alpha$, where
$\alpha\in(0,1)$.

Assume $Q\ll P$ and write $w^\star=dQ/dP$. Given features
$\phi\colon\X\times\Y\to\R^d$, the exponential tilt model is
\begin{equation}
 h_\beta(x,y)=e^{\beta^\top\phi(x,y)},\qquad
 w_\beta(x,y)=\frac{h_\beta(x,y)}{Z_P(\beta)},\qquad
 Z_P(\beta)=\E_P h_\beta(X,Y).
 \label{eq:ratio-model}
\end{equation}
We restrict $\beta\in\mathcal B$ to $0<Z_P(\beta)<\infty$ and define
$dQ_\beta=w_\beta\,dP$. Features may be learned from $\Dtr$, then
frozen. A common constant in the log tilt cancels under normalization
and is removed when present. A nonconstant input term remains in the
joint ratio.

For classification with $K$ classes, let
$T\colon\X\to\R^p$ be an input feature map. We construct
$\phi(x,y)$ by stacking the class indicators
$\mathbf{1}\{y=k\}$ and their products
$\mathbf{1}\{y=k\}T(x)$, for $k=1,\ldots,K$.
With corresponding coefficients $\alpha_k$ and $\theta_k$,
\[
 \beta^\top\phi(x,y)
 =\sum_{k=1}^K \mathbf{1}\{y=k\}
   \bigl(\alpha_k+\theta_k^\top T(x)\bigr).
\]
Thus, for a candidate class $k$, $h_\beta(x,k)=e^{\theta_k^\top T(x)+\alpha_k}$,
recovering the class-specific exponential tilts of
\citet{MaityS2023iclr}.
For regression, we choose features of the continuous response and its
interactions with the input. Input-only features describe covariate
shift and response-only features describe label or target shift.
Correct specification means $w^\star=w_{\beta^\star}$ for some
$\beta^\star\in\mathcal B$. The identification calculations and
synthetic designs use this assumption. The coverage bound below also
allows misspecification.

\subsection{Marginal matching}
\label{subsec:extra-fit}

The joint tilt model specifies
\[
 Q_\beta(dx,dy)
 =\frac{h_\beta(x,y)}{Z_P(\beta)}P(dx,dy).
\]
To connect this model to the observed target inputs, we integrate out
the response. Using $P(dx,dy)=P_X(dx)P(dy\mid x)$, where $P(dy\mid x)$
is the source conditional distribution of $Y$ given $X=x$, gives
\[
 Q_{\beta,X}(dx)
 =\int_{\Y}Q_\beta(dx,dy)
 =\frac{P_X(dx)}{Z_P(\beta)}
   \int_{\Y}h_\beta(x,y)P(dy\mid x).
\]
Define the conditional moment
\begin{equation}
 M_\beta(x)
 =\int_{\Y}h_\beta(x,y)P(dy\mid x)
 =\E_P[h_\beta(X,Y)\mid X=x].
 \label{eq:conditional-moment-population}
\end{equation}
The induced input ratio is
\begin{equation}
 \frac{dQ_{\beta,X}}{dP_X}(x)
 =\frac{M_\beta(x)}{Z_P(\beta)},
 \qquad
 Z_P(\beta)=\E_{P_X}M_\beta(X).
 \label{eq:induced-input-ratio}
\end{equation}
Thus, averaging the joint tilt over the source conditional distribution
produces an input weight $M_\beta(x)$. Normalizing these weights
determines the input distribution implied by the proposed joint shift.

ExTRA estimates $\beta$ by matching this induced distribution to
the observed target input distribution $Q_X$.
Its population objective is
\begin{equation}
 L(\beta)
 =\E_{Q_X} \left[ \log\frac{dQ_{\beta,X}}{dP_X}(X) \right]
 =\E_{Q_X} \left[ \log M_\beta(X) \right]-\log Z_P(\beta).
 \label{eq:population-objective}
\end{equation}
When the quantities are finite,
\[
 \operatorname{KL}(Q_X\Vert Q_{\beta,X})
 =\operatorname{KL}(Q_X\Vert P_X)-L(\beta).
\]
Maximizing $L$ therefore minimizes the discrepancy between the
induced and observed input distributions, without estimating a
target density. Under correct specification, every parameter
reproducing $Q_X$ maximizes this criterion. Different joint tilts
may nevertheless induce the same input marginal. Whether matching
$Q_X$ determines the joint ratio is the identification question
considered below.

In practice, let $\widehat P(dy\mid x)$ denote the source conditional
distribution fitted on $\Dtr$, with density or mass function
$\widehat p_P(y\mid x)$. We approximate $M_\beta(x)$ by
\begin{equation}
 \widehat M_\beta(x)
 =\int_{\Y}h_\beta(x,y)\widehat P(dy\mid x).
 \label{eq:conditional-moment}
\end{equation}
For classification, this integral is a sum over labels, weighted
by their fitted source probabilities. For regression, it can be
evaluated analytically, by quadrature, or by sampling from the
fitted conditional distribution.

The target shift sample $\Dshift^Q$ estimates the expectation in
$L(\beta)$, while the labeled source shift sample $\Dshift^P$
estimates $Z_P(\beta)=\E_P h_\beta(X,Y)$.
With a ridge penalty $(\lambda/2)\|\beta\|_2^2$, where
$\lambda\geq0$, we estimate $\beta$ by maximizing
\begin{equation}
 \widehat L(\beta)
 =\frac1m\sum_{j=1}^m\log\widehat M_\beta(X_j^Q)
 -\log\!\left\{
   \frac1{n_{\mathrm{shift}}}
   \sum_{i=1}^{n_{\mathrm{shift}}}h_\beta(X_i^P,Y_i^P)
 \right\}
 -\frac{\lambda}{2}\|\beta\|_2^2.
 \label{eq:extra-objective}
\end{equation}
We maximize $\widehat L(\beta)$ numerically over $\mathcal B$,
subject to prespecified coefficient bounds and finite positive
fitted moments. Appendix~\ref{app:implementation} reports the
bounds and initialization choices for each experiment.
Numerical convergence does not guarantee a global maximum.

The fitted conditional moments and empirical normalizer need not
satisfy the population identity $Z_P(\beta)=\E_{P_X}M_\beta(X)$.
The moments use the learned source
conditional distribution, while the normalizer uses observed
source pairs. We therefore treat $\widehat L$ as a penalized
plug-in approximation to the population objective.
Source-model error, sampling variability, regularization, and
numerical optimization can all affect $\widehat\beta$ and hence
the fitted joint ratio. We assess this error in the synthetic
experiments. We do not derive a general regression estimation rate.

\subsection{What target inputs identify}
\label{subsec:identification}

Matching the target input distribution does not necessarily determine
the joint shift. Identification asks whether that match uniquely
determines the joint ratio within the assumed family.
For fixed $P$, the ratio is identified at $\beta^\star$ if, for
every $\beta\in\mathcal B$,
\begin{equation}
 Q_{\beta,X}=Q_{\beta^\star,X}
 \quad\Longrightarrow\quad
 w_\beta=w_{\beta^\star}
 \quad P\text{ almost surely}.
 \label{eq:identification-condition}
\end{equation}
By \eqref{eq:induced-input-ratio}, equality of the input marginals
is equivalent to
\[
 M_\beta(x)=C M_{\beta^\star}(x)
 \quad P_X\text{ almost surely},
 \qquad
 C=\frac{Z_P(\beta)}{Z_P(\beta^\star)}.
\]
The question is whether this proportionality also determines the
normalized joint ratio. \citet[Section~4.1]{MaityS2023iclr}
discuss label-switching ambiguities and sufficient conditions
for identification in classification.

Under \emph{nonidentification}, different joint ratios produce
exactly the same input distribution. Even unlimited labeled source
data and unlabeled target inputs cannot distinguish them.
A penalty can select one solution, but cannot supply the missing
information. For example, when $X$ and $Y$ are independent under
$P$, a response-only tilt leaves the input distribution unchanged.
Appendix~\ref{app:identification} shows that ambiguity can also
arise when source inputs are informative about the response.

We use \emph{weak identification} to describe settings where
substantially different joint ratios produce input distributions
that are difficult to distinguish at the available sample sizes.
Here the ratio may be identified in the population but estimated
poorly from finite samples. We therefore distinguish three
questions.
\begin{enumerate}
\item \emph{Shift representation.}
Does the chosen family contain the true joint ratio?
\item \emph{Identification.}
Do target inputs uniquely determine the joint ratio within
that family?
\item \emph{Estimation accuracy.}
How accurately does the fitted estimator recover the joint
ratio from the available samples?
\end{enumerate}
The experiments examine how each issue affects weighted
calibration and predictive tilting.

\section{Weighted calibration and predictive tilting}
\label{sec:methods}

ExTRA-WCP and ExTRA-WCP-T use the same fitted joint ratio to assign weights to
the calibration pairs and to each candidate pair $(x,y)$ when computing the conformal rank.
They differ in how they compute scores. ExTRA-WCP retains
the source score, while ExTRA-WCP-T first tilts the source predictive
distribution and then computes the score from that adjusted
distribution. Figure~\ref{fig:schematic} shows the shared fitting
steps and the additional predictive adjustment in ExTRA-WCP-T.

\subsection{Two scores from the same fitted model}

Let $p_0(y\mid x)$ denote the source conditional density or mass
function fitted on $\Dtr$. A scoring rule converts this predictive
distribution into a score $S_0(x,y)$, with larger values indicating
less conformity. Examples include negative log density and adaptive
prediction set (APS) scores \citep{RomanoY2020neurips}.
ExTRA-WCP retains this source score.

To construct the alternative score, first consider the conditional
distribution implied by the joint tilt. Dividing the tilted joint
distribution by its input marginal gives
\begin{equation}
 Q_\beta(dy\mid x)
 =\frac{h_\beta(x,y)P(dy\mid x)/Z_P(\beta)}
        {M_\beta(x)/Z_P(\beta)}
 =\frac{h_\beta(x,y)}{M_\beta(x)}P(dy\mid x).
 \label{eq:tilt-induced-conditional}
\end{equation}
The global normalizer $Z_P(\beta)$ cancels, leaving the conditional
normalizer $M_\beta(x)$. In terms of conditional densities or mass
functions, this relation is
$q_\beta(y\mid x)=h_\beta(x,y)p_P(y\mid x)/M_\beta(x)$,
where $p_P$ denotes the true source conditional density or mass function.

ExTRA-WCP-T applies this transformation to the fitted predictor, giving
\begin{equation}
 p_\beta(y\mid x)
 =\frac{h_\beta(x,y)p_0(y\mid x)}
        {Z_\beta^{\mathrm{pred}}(x)},
 \qquad
 Z_\beta^{\mathrm{pred}}(x)
 =\int h_\beta(x,y)p_0(y\mid x)\,dy.
 \label{eq:tilted-predictive}
\end{equation}
For classification, the integral is a sum over labels.
The normalizer ensures that $p_\beta(\cdot\mid x)$ integrates
to one at each input. In our experiments, $p_0=\widehat p_P$,
so $Z_\beta^{\mathrm{pred}}(x)=\widehat M_\beta(x)$.
This fitted normalizer generally differs from the population
moment $M_\beta(x)$.

We apply the same scoring rule to $p_0$ and $p_\beta$.
For negative log density, $S_0(x,y)=-\log p_0(y\mid x)$ and
\begin{equation}
 S_\beta(x,y)
 =S_0(x,y)-\log h_\beta(x,y)
           +\log Z_\beta^{\mathrm{pred}}(x).
 \label{eq:tilted-score}
\end{equation}
ExTRA-WCP-T uses $S_{\widehat\beta}$ in weighted calibration.
Although $\log Z_\beta^{\mathrm{pred}}(x)$ is constant across
responses at a fixed input, it varies across inputs and must
be retained when comparing calibration and candidate scores.

An input-only tilt cancels from \eqref{eq:tilted-predictive},
leaving the predictive distribution and score unchanged.
A response-dependent tilt can change the relative probabilities
of candidate responses without refitting the source model.
This adjustment applies to any suitable predictive distribution
and does not require a Bayesian model.

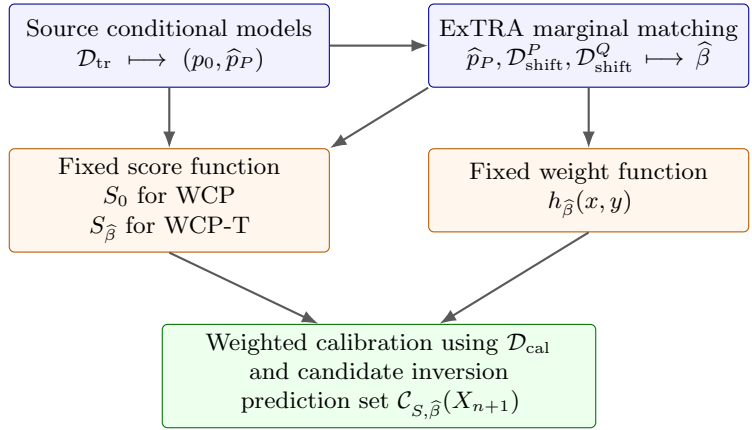
\begin{figure}[htbp]
\centering
\begin{tikzpicture}[
  >={Latex[length=2.2mm]},
  node distance=8mm and 13mm,
  input/.style={draw=blue!55!black,fill=blue!5,rounded corners=2pt,
    align=center,minimum height=11mm,text width=4.0cm,font=\small},
  role/.style={draw=orange!70!black,fill=orange!7,rounded corners=2pt,
    align=center,minimum height=11mm,text width=4.0cm,font=\small},
  output/.style={draw=green!45!black,fill=green!7,rounded corners=2pt,
    align=center,minimum height=12mm,text width=5.5cm,font=\small},
  flow/.style={->,thick,draw=black!65}
]
\node[input] (predictor) {Source conditional models\\
  $\Dtr\longmapsto(p_0,\widehat p_P)$};
\node[input, right=of predictor] (extra) {ExTRA marginal matching\\
  $\widehat p_P,\Dshift^P,\Dshift^Q
  \longmapsto\widehat\beta$};
\node[role, below=of predictor] (score) {Fixed score function\\
  $S_0$ for WCP\\$S_{\widehat\beta}$ for WCP-T};
\node[role, below=of extra] (weights) {Fixed weight function\\
  $h_{\widehat\beta}(x,y)$};
\node[output, below=17mm of $(score)!0.5!(weights)$] (set)
  {Weighted calibration using $\Dcal$\\
   and candidate inversion\\
   prediction set $\C_{S,\widehat\beta}(X_{n+1})$};
\draw[flow] (predictor) -- (score);
\draw[flow] (predictor.east) -- (extra.west);
\draw[flow] (extra) -- (weights);
\draw[flow] (extra.south west) -- (score.north east);
\draw[flow] (score.south) -- ([xshift=-8mm]set.north);
\draw[flow] (weights.south) -- ([xshift=8mm]set.north);
\end{tikzpicture}
\caption{ExTRA-WCP and ExTRA-WCP-T share the fitted joint ratio and
calibration weights. Predictive tilting changes only the score. The
source conditional estimate $\widehat p_P$ also supplies the conditional
moments used to fit ExTRA. All fitted functions are fixed before calibration.}
\label{fig:schematic}
\end{figure}

\paragraph{Admissibility and source-model error.}
The fitted tilt, $h_{\widehat{\beta}}(x,y) = e^{\widehat{\beta}^{\top} \phi(x,y)}$,
must satisfy two integrability conditions.
The joint normalizer must obey
$0<Z_P(\widehat\beta)<\infty$ so that the tilt defines a joint
density ratio. Predictive tilting additionally requires
$0<Z_{\widehat\beta}^{\mathrm{pred}}(x)<\infty$
for $P_X$ almost every $x$ so that the adjusted predictor is
well defined. When $p_0\ne p_P$, joint integrability does not
by itself ensure predictive integrability. If only the latter
fails, ExTRA-WCP can still use the source score with the same
joint weights. If joint integrability fails, the fitted tilt
must be replaced before calibration.

Even with the true tilt, the adjusted predictor need not equal
the target conditional distribution. To see this, assume correct
specification, $Q=Q_{\beta^\star}$, and write $q(y\mid x)$ for
the true target conditional density or mass function. At a fixed
input $x$, suppose $p_0(\cdot\mid x)\ll p_P(\cdot\mid x)$ and define
$r_x(y)=p_0(y\mid x)/p_P(y\mid x)$. With finite positive normalizers,
\begin{equation}
 \frac{p_{\beta^\star}(y\mid x)}{q(y\mid x)}
 =\frac{r_x(y)}{\E_{Q_{Y\mid X=x}}r_x(Y)}
 \quad Q_{Y\mid X=x}\text{ almost surely}.
 \label{eq:predictive-mismatch}
\end{equation}
Thus, the discrepancy between the adjusted predictor and the
target is the relative source-model error, rescaled by its
target conditional mean. The true tilt does not generally
remove that error. Nevertheless, true joint weights guarantee
target marginal coverage for a well-defined score fitted
independently of calibration and testing. Coverage validity
therefore does not require an exact predictive distribution.

\subsection{Candidate-dependent weighted calibration}

Fix a score $S$ and an estimated tilt function $h_{\widehat\beta}$
before calibration. For the source calibration pairs
$(X_i,Y_i)$, $i=1,\ldots,n$, define
\[
 V_i=S(X_i,Y_i),
 \qquad
 H_i=h_{\widehat\beta}(X_i,Y_i).
\]
At a target input $x=X_{n+1}$, each candidate response $y$ has
its own score and weight,
\[
 V_{n+1}(y)=S(x,y),
 \qquad
 H_{n+1}(y)=h_{\widehat\beta}(x,y).
\]
We compare the candidate score with the calibration scores using
the weighted upper rank
\begin{equation}
 \widehat p(y)=
 \frac{H_{n+1}(y)+
       \sum_{i=1}^n H_i\ind\{V_i\geq V_{n+1}(y)\}}
      {H_{n+1}(y)+\sum_{i=1}^n H_i}.
 \label{eq:weighted-p}
\end{equation}
The numerator includes the candidate's own weight and the weights
of calibration observations whose scores are at least as large.
Including tied scores makes the rule conservative.
The prediction set retains candidates whose weighted rank
exceeds $\alpha$, giving
\begin{equation}
 \C_{S,\widehat\beta}(x)
 =\{y\in\Y:\widehat p(y)>\alpha\}.
 \label{eq:prediction-set}
\end{equation}
ExTRA-WCP uses $S=S_0$, while ExTRA-WCP-T uses
$S=S_{\widehat\beta}$. Both use the same weights.

The global normalizer $Z_P(\widehat\beta)$ cancels from
\eqref{eq:weighted-p}, so the weights can be computed directly
from $h_{\widehat\beta}$. By contrast, the predictive normalizer
$Z_{\widehat\beta}^{\mathrm{pred}}(x)$ remains part of the tilted
score. Dividing each calibration weight by its predictive normalizer
would generally change the
joint ratio and the implied target distribution.

When the fitted ratio depends on the response, changing $y$
can change both the candidate score and its share of the total
weight. The corresponding weighted score threshold can therefore
depend on $y$, so a single cutoff need not describe the entire
prediction set. When the fitted weights depend only on inputs,
the candidate weight is constant across responses, recovering
the usual covariate-weighted construction.

For classification, we evaluate \eqref{eq:weighted-p} for every
label. For regression, we invert the same rule over the full
response space. Randomized scores use independent auxiliary
variables with the same distribution under source and target.
Coverage averages over this randomness as well as the sampled data.

\FloatBarrier

\section{Coverage and the role of the score}
\label{sec:theory}

Both procedures use the same estimated joint ratio and therefore
share a coverage bound. Their actual coverage can nevertheless
differ because their scores determine which responses the sets miss.

\subsection{A common coverage bound}
\label{subsec:coverage}

Let $\mathcal F$ contain the fitting samples and fitting randomness.
Conditional on $\mathcal F$, the weight function and scoring rule
are fixed, the calibration pairs follow $P^n$, and the test pair
is independent of calibration. Any auxiliary randomness used to
evaluate scores is independent of these fitting quantities and
is included in the coverage probabilities below.

Assume $0<\E_P h_{\widehat\beta}<\infty$. The estimated tilt function
defines the normalized joint ratio and surrogate distribution
\begin{equation}
 \widehat w
 =\frac{h_{\widehat\beta}}{\E_P h_{\widehat\beta}},
 \qquad
 d\widehat Q=\widehat w\,dP.
 \label{eq:surrogate-target}
\end{equation}
We call $\widehat Q$ the surrogate target. The population normalizer
defines this distribution but need not be computed, since it cancels
from the weighted rank.

Write $\C_{S,\widehat w}$ for the prediction set in
\eqref{eq:prediction-set} constructed with score $S$ and weights
$\widehat w$. Weighted exchangeability
\citep[Lemmas~2--3 and Theorem~2]{TibshiraniR2019neurips},
applied to the joint observations, gives
\begin{equation}
 \Pp_{P^n\times\widehat Q}
 \{Y_{n+1}\in\C_{S,\widehat w}(X_{n+1})\mid\mathcal F\}
 \geq1-\alpha.
 \label{eq:surrogate-coverage}
\end{equation}
Thus, the fitted weights guarantee coverage under the distribution
they define. To obtain a guarantee under the actual target $Q$,
keep the calibration distribution fixed and change only the test
distribution from $\widehat Q$ to $Q$. Total variation (TV) bounds
the resulting change in coverage, giving
\begin{equation}
 \begin{aligned}
 &\Pp_{P^n\times Q}
 \{Y_{n+1}\in\C_{S,\widehat w}(X_{n+1})\mid\mathcal F\}
 \geq1-\alpha-\TV(Q,\widehat Q),\\
 &\TV(Q,\widehat Q)
 =\tfrac12\E_P|w^\star-\widehat w|.
 \end{aligned}
 \label{eq:target-transfer}
\end{equation}
This applies the estimated-weight argument of
\citet[Proposition~1]{LeiL2021jrsssb} to joint ratios.
See also \citet{BarberRF2023aos} for total variation bounds
beyond exchangeability.

The bound holds for either score, including when the source
predictor or ratio family is misspecified. With true weights,
both methods attain coverage at least $1-\alpha$. With estimated
weights, they share the same lower bound because they share
$\widehat w$. These guarantees average over calibration, the
test pair, and score randomness. They do not guarantee coverage
at every input or for every realized calibration sample.

\subsection{Why the scores can have different coverage}
\label{subsec:alignment}

ExTRA-WCP and ExTRA-WCP-T share the same estimated weights and
coverage bound, but their actual coverage can differ. Changing the
score changes which responses the prediction set tends to miss.
The effect of ratio estimation error therefore depends on where
those misses occur.

For a fixed pair $(x,y)$, define
\[
 e_S(x,y)
 =\Pp\{y\notin\C_{S,\widehat w}(x)\mid\mathcal F\},
\]
where the probability averages over source calibration samples
and auxiliary score randomness. Thus, $e_S(x,y)$ is the probability
of missing response $y$ at input $x$, conditional on the fitted
functions. Write $\mathrm{miss}$ for the event
$Y_{n+1}\notin\C_{S,\widehat w}(X_{n+1})$.
Keeping calibration unchanged and replacing the surrogate test
distribution $\widehat Q$ with the actual target $Q$ gives
\begin{equation}
 \begin{aligned}
 \Delta_S
 &:=\Pp_{P^n\times Q}(\mathrm{miss}\mid\mathcal F)
   -\Pp_{P^n\times\widehat Q}(\mathrm{miss}\mid\mathcal F)\\
 &=\E_P[(w^\star-\widehat w)e_S].
 \end{aligned}
 \label{eq:score-transfer-identity}
\end{equation}
Where $\widehat w$ understates target probability, frequent misses
contribute positively to $\Delta_S$. Where the procedure almost
always covers, the same ratio error contributes little.
Since $0\leq e_S\leq1$, every score satisfies
$|\Delta_S|\leq\TV(Q,\widehat Q)$, while its actual transfer
error depends on the miss probabilities $e_S$.

A positive transfer error does not by itself imply coverage below
nominal, because calibration under the surrogate may be conservative.
Define its surplus coverage by
\[
 s_S
 =\alpha-
 \Pp_{P^n\times\widehat Q}(\mathrm{miss}\mid\mathcal F)
 \geq0.
\]
Then
\[
 \Pp_{P^n\times Q}(\mathrm{miss}\mid\mathcal F)-\alpha
 =\Delta_S-s_S.
\]
Target undercoverage occurs when the transfer error exceeds this
surplus. Differences in target coverage can therefore reflect both
the alignment of ratio errors with misses and differences in
conservatism under the surrogate. Observed target coverage alone
does not separate these effects.

For the regression families studied below, the transfer error
can be separated into two response modes. Let
$Z=\operatorname{sign}(Y)$, assume $P(Y=0)=0$, and suppose
both ratios depend only on $(x,z)$. Denote their values by
$w_z^\star(x)$ and $\widehat w_z(x)$, and define
\[
 \pi_z(x)=P(Z=z\mid X=x),
 \qquad
 m_{S,z}(x)
 =\E_P[e_S(x,Y)\mid X=x,Z=z].
\]
Here $m_{S,z}(x)$ averages the miss probability over responses
within mode $z$ at input $x$. Since the weights are constant
within each mode, they preserve the conditional distribution
of $Y$ given $(X,Z)$. Conditioning
\eqref{eq:score-transfer-identity} on $(X,Z)$ yields
\begin{equation}
 \Delta_S
 =\E_{P_X}\sum_{z=\pm1}
 \pi_z(X)\{w_z^\star(X)-\widehat w_z(X)\}m_{S,z}(X).
 \label{eq:mode-transfer}
\end{equation}

The role of mode probabilities is especially clear when the
input marginals match exactly, $Q_X=\widehat Q_X$. Write
\[
 \tau^\star(x)=Q(Z=+1\mid X=x),
 \qquad
 \widehat\tau(x)=\widehat Q(Z=+1\mid X=x).
\]
Then \eqref{eq:mode-transfer} simplifies to
\begin{equation}
 \Delta_S
 =\E_{Q_X}
 [(\tau^\star-\widehat\tau)(m_{S,+}-m_{S,-})].
 \label{eq:mode-alignment}
\end{equation}
For example, the integrand is positive when the surrogate
understates the positive-mode probability and the procedure
misses more often on that mode. More generally, a nonnegative
product $Q_X$ almost everywhere, positive on a set of positive
$Q_X$ measure, implies $\Delta_S>0$. Undercoverage additionally
requires $\Delta_S>s_S$.

When input marginals do not match, \eqref{eq:mode-transfer}
also accounts for their discrepancy. The mode probability
$\widehat\tau$ is defined under the surrogate $\widehat Q=\widehat wP$.
It need not equal the mode probability of the tilted fitted
predictor, which also depends on source-model error.

The score adjustment shows how tilting changes the relative
treatment of the two modes. For $x=(u,v)$ and
$h_{a,b}(x,y)=e^{au+b\operatorname{sign}(y)}$, let
$\pi_z^0(x)$ denote the source predictor's probability of mode $z$.
Negative log density gives
\begin{equation}
 S_{a,b}(x,y)
 =S_0(x,y)-bz+
 \log\{\pi_+^0(x)e^b+\pi_-^0(x)e^{-b}\},
 \qquad z=\operatorname{sign}(y).
 \label{eq:sign-score}
\end{equation}
The input-only factor $e^{au}$ cancels. At a fixed input,
the positive-mode score adjustment minus the negative-mode
adjustment is $-2b$. Thus, $b>0$ favors the positive mode
relative to the source score. For the interaction tilt
$e^{au+(b+cu)z}$, the difference becomes $-2(b+cu)$, allowing
the adjustment to vary with the input.

This score comparison alone does not determine coverage.
As the tilt coefficients vary, calibration scores, weights,
and candidate thresholds also change.
Appendix~\ref{app:inversion} gives the full set inversion
and explicit within-mode miss probabilities for the
regression model.

\paragraph{Checking the full response set.}
\label{subsec:geometry}

Set size requires a separate check because a candidate's own
weight can force inclusion regardless of its score. Fix an
input $x$, let $W=\sum_{i=1}^nH_i>0$, and write
$H(y)=h_{\widehat\beta}(x,y)>0$, with all weights finite.
Equation~\eqref{eq:weighted-p} implies
\begin{equation}
 \widehat p(y)\geq\frac{H(y)}{W+H(y)},
 \qquad
 H(y)>\frac{\alpha W}{1-\alpha}
 \ \Longrightarrow\
 y\in\C_{S,\widehat w}(x)
 \quad\text{for every }S.
 \label{eq:forced-inclusion-threshold}
\end{equation}
This is the effect of the candidate atom in the weighted
quantile construction \citep{TibshiraniR2019neurips}.
The strict inequality corresponds to the inclusion rule
$\widehat p(y)>\alpha$.

For example, on $\Y=\R$, a linear response tilt
$H(y)=e^{r(x)+\gamma y}$ with finite $r(x)$ and $\gamma\ne0$
eventually exceeds this threshold throughout one response tail.
Both score choices therefore produce sets of infinite total
length at that input. Yet the predictive normalizer can remain
finite, as it does for a Gaussian predictor under this tilt.

Bounded response features prevent candidate weights from
diverging at a fixed input, but those weights may still exceed
the inclusion threshold. Predictive normalizability alone
therefore does not ensure finite prediction-set length.
Regression sets must be checked over the full response space.

\section{Experiments}
\label{sec:experiments}

We compare coverage and set size at nominal coverage $0.90$. Each
WCP/WCP-T pair shares the learned predictor, fitted ratio, calibration
sample, and test observations. Target responses are used only for
evaluation within each run. We compute target coverage and average set
size per run, using class count for classification and the sum of
component lengths for regression, allowing disconnected sets. Tables
report the means of these run-level quantities with standard errors
(SEs). Figures show means $\pm$ one SE across replications or real-data
source partitions. Paired $95\%$
Student's $t$ intervals are descriptive, without multiplicity adjustment.
Size comparisons are interpreted alongside achieved coverage.

Each synthetic setting uses $30$ independent replications. Each
replication draws five independent samples comprising $2000$ source
training pairs, $2000$ source shift pairs, $1500$ source calibration
pairs, $5000$ target inputs, and $4000$ target test pairs. Scores are negative log conditional probability or
density. Appendix~\ref{app:implementation} gives fitting details.

\subsection{Synthetic classification}
\label{subsec:phase}
The source has three Gaussian classes with probabilities
$(0.36,0.34,0.30)$, covariance $I_2$, and means $(-2,0)^\top$,
$(2,0)^\top$, and $(0,2.5)^\top$. Class-specific linear exponential
tilts change both class proportions and within-class inputs. A fitted
multinomial logistic regression supplies scores and ExTRA moments.
Coefficient bounds exclude alternative component assignments, identifying
the joint ratio at this target.

\begin{table}[ht!]
\centering
\caption{Synthetic classification at nominal coverage $0.90$. Entries
are means (SEs) over $30$ replications. Oracle rows use the true joint
ratio while retaining the same learned predictor, calibration sample,
and test observations.}
\label{tab:classification}
\small
\begin{tabular}{llrr}
\hline
Ratio & Method & Coverage & Set size\\
\hline
None & Standard CP & 0.775 (0.003) & 0.914 (0.004)\\
Fitted & ExTRA-WCP & 0.924 (0.007) & 1.445 (0.062)\\
Fitted & ExTRA-WCP-T & 0.814 (0.020) & 1.219 (0.017)\\
Oracle & ExTRA-WCP & 0.903 (0.003) & 1.246 (0.012)\\
Oracle & ExTRA-WCP-T & 0.903 (0.003) & 1.202 (0.010)\\
\hline
\end{tabular}
\end{table}

ExTRA-WCP raises coverage from standard CP's $0.775$ to $0.924$, but
tilting lowers it to $0.814$ (Table~\ref{tab:classification}). The paired
WCP-T minus WCP difference is $-0.110$ ($95\%$ interval
$[-0.164,-0.057]$). True weights bring both methods to $0.903$.
Identification therefore does not ensure estimation accurate enough for
tilting. This control does not separate source-model, normalization,
regularization, and optimization errors. The coverage loss limits any
efficiency claim for the smaller tilted sets. Standard CP's mean size
is below one because $8.56\%$ of its sets are empty, compared with
$0\%$ for fitted WCP and $0.19\%$ for WCP-T.

\subsection{Regression under covariate and response shifts}
\label{subsec:regression-adaptation}

We design this experiment around two changes that can make prediction
harder after deployment. The target may contain noisier inputs, and
outcomes that were uncommon in the source may become more frequent.
Varying these changes separately and together lets us ask whether joint
weighting restores coverage and whether tilting reduces the length
needed to obtain that coverage.

The response has two modes, concentrated around negative and positive
values. Let $X=(U,V)$, where $U$ and $V$ are independent and uniform
on $[-1,1]$, and let $Z=\operatorname{sign}(Y)$ indicate the mode.
We generate source observations from
\begin{align}
 P_\eta(Z=+1\mid u,v)
 &=\operatorname{logit}^{-1}(d_\eta+3\eta v),\nonumber\\
 Y\mid u,v,z
 &\sim\mathcal N\!\left(1.25z+0.10v,\{0.22e^{0.8u}\}^2\right)
              \quad\text{restricted to }\{zy>0\},
 \label{eq:adaptation-source}
\end{align}
where $\operatorname{logit}^{-1}(t)=1/(1+e^{-t})$ is the logistic function. 
Larger $U$ increases
the noise scale, while $V$ indicates which mode is more likely.
Truncation to the corresponding half-line makes the mode observable
from the sign of a source response. Thus, $U$ controls response
variability, while $V$ provides information about response-mode frequencies.

Here $\eta=1$ and $d_1=-2$. Averaging the conditional positive-mode
probability over $V$ gives
\[
p_+=\E_V[P(Y>0\mid V)]
 =\frac12\int_{-1}^1\operatorname{logit}^{-1}(-2+3v)\,dv
 =\frac{\log(1+e)-\log(1+e^{-5})}{6}\simeq0.218.
\]

Positive responses thus account for about $22\%$ of the source
population. We generate targets by normalizing the tilt
\begin{equation}
 h_{a,b}(u,v,y)=\exp\{au+b\operatorname{sign}(y)\},\qquad
 (a^\star,b^\star)\in\{0,1\}\times\{0,0.6,1.2\}.
 \label{eq:adaptation-tilt}
\end{equation}
We use $(a^\star,b^\star)$ for the coefficients generating a target
and $(\widehat a,\widehat b)$ for their estimates. At $a^\star=1$,
larger values of $U$ receive more weight, shifting the population toward
noisier inputs. Positive $b^\star$ increases the frequency of the
initially rare positive mode. Since $U$ and $Z$ are source-independent,
\[
 Q(Z=+1)=\frac{p_+e^{b^\star}}
                   {p_+e^{b^\star}+(1-p_+)e^{-b^\star}},
\]
which reaches $0.754$ at $b^\star=1.2$. The six settings include
the no-shift control $(0,0)$, the covariate-shift control $(1,0)$,
two response-only shifts with $a^\star=0$, and two combined shifts.
The factorization $h_{a,b}=e^{au}e^{bz}$ makes the shift
\emph{separable}. It changes the input distribution and mode
frequencies while preserving $Y\mid U,V,Z$.

The role of $V$ is essential because target responses are unavailable
during fitting. Favoring the positive mode also favors the values of
$V$ where that mode is common in the source. This change in the target
input distribution identifies $b$ when $\eta>0$
(Proposition~\ref{prop:separated-identification}). The distribution of
$U$ separately identifies $a$, because $U$ and $Z$ are independent
under the source. Nonconstancy of the mode probability supplies
identification. Separating its input from the covariate tilt makes
the two parameters easier to estimate in this design.

The feature restriction is nevertheless substantive. Allowing an
additional $cv\operatorname{sign}(y)$ term gives $(a,b,c)$ and
$(a,-d_\eta-b,-3\eta-c)$ the same input marginal despite complementary
conditional mode probabilities (Proposition~\ref{prop:learned-ambiguity}).
Matching target inputs cannot establish that this interaction is absent.

Both the source conditional and tilt families are correctly specified,
but all coefficients are learned. We fit a logistic mode probability
on $(u,v)$, per-sign Gaussian locations linear in $(1,u,v)$, and a
shared log scale linear in $u$. Both $a$ and $b$ are fitted in every
case, including controls. This favorable specification allows us to
study weighting and scoring when the model can represent the target.
A two-stage baseline first estimates the
covariate tilt on $U$, then uses soft black-box shift estimation (BBSE)
\citep{LiptonZ2018icml}
to estimate the binary mode prior from the shared learned gate.
BBSE uses the binary source sign as its label.
Appendix~\ref{app:two-stage} explains the construction.

Pilot studies guided the design. All settings and both ExTRA penalties
were fixed before final evaluation. Settings share source samples and
predictors within a replication, with independent target samples.
The two-stage baseline was added later with fitting choices fixed
before its evaluation. Final target responses were never used for tuning.

\begin{figure}[ht!]
\centering
\begin{tikzpicture}
\begin{axis}[resultaxis,height=4.5cm,title={Response shift ($a^\star=0$)},ylabel={Coverage},
  xmin=-.2,xmax=2.2,xtick={0,1,2},xticklabels={{$(0,0)$},{$(0,0.6)$},{$(0,1.2)$}},
  xlabel={Shift parameters $(a^\star,b^\star)$},ymin=.72,ymax=.94,ytick={.75,.80,.85,.90},legend columns=3,legend to name=adaptlegend]
\addplot[cpplot,error bars/.cd,y dir=both,y explicit] coordinates {
(-0.06,0.898642) +- (0,0.001617)
(0.94,0.851958) +- (0,0.002176)
(1.94,0.803683) +- (0,0.002705)
};
\addlegendentry{Standard CP}
\addplot[wcpplot,error bars/.cd,y dir=both,y explicit] coordinates {
(0.00,0.900650) +- (0,0.002344)
(1.00,0.900375) +- (0,0.002360)
(2.00,0.900617) +- (0,0.002470)
};
\addlegendentry{ExTRA-WCP}
\addplot[tiltplot,error bars/.cd,y dir=both,y explicit] coordinates {
(0.06,0.899258) +- (0,0.002179)
(1.06,0.898950) +- (0,0.001704)
(2.06,0.897925) +- (0,0.002299)
};
\addlegendentry{ExTRA-WCP-T}
\addplot[black!65,dashed,no marks,forget plot] coordinates {(-.2,.9)(2.2,.9)};
\end{axis}
\end{tikzpicture}\hfill
\begin{tikzpicture}
\begin{axis}[resultaxis,height=4.5cm,title={Covariate and response shift ($a^\star=1$)},ylabel={Coverage},
  xmin=-.2,xmax=2.2,xtick={0,1,2},xticklabels={{$(1,0)$},{$(1,0.6)$},{$(1,1.2)$}},
  xlabel={Shift parameters $(a^\star,b^\star)$},ymin=.72,ymax=.94,ytick={.75,.80,.85,.90}]
\addplot[cpplot,error bars/.cd,y dir=both,y explicit] coordinates {
(-0.06,0.870133) +- (0,0.001756)
(0.94,0.814025) +- (0,0.002793)
(1.94,0.755450) +- (0,0.003137)
};
\addplot[wcpplot,error bars/.cd,y dir=both,y explicit] coordinates {
(0.00,0.900650) +- (0,0.002457)
(1.00,0.902925) +- (0,0.003083)
(2.00,0.903783) +- (0,0.003032)
};
\addplot[tiltplot,error bars/.cd,y dir=both,y explicit] coordinates {
(0.06,0.897492) +- (0,0.002160)
(1.06,0.901833) +- (0,0.002623)
(2.06,0.898608) +- (0,0.003221)
};
\addplot[black!65,dashed,no marks,forget plot] coordinates {(-.2,.9)(2.2,.9)};
\end{axis}
\end{tikzpicture}
\par\smallskip
\begin{tikzpicture}
\begin{axis}[resultaxis,height=4.5cm,title={Response shift ($a^\star=0$)},ylabel={Mean total length},
  xmin=-.2,xmax=2.2,xtick={0,1,2},xticklabels={{$(0,0)$},{$(0,0.6)$},{$(0,1.2)$}},
  xlabel={Shift parameters $(a^\star,b^\star)$},ymin=1.0,ymax=2.35,ytick={1.0,1.5,2.0}]
\addplot[cpplot,error bars/.cd,y dir=both,y explicit] coordinates {
(-0.06,1.157964) +- (0,0.006396)
(0.94,1.230118) +- (0,0.007246)
(1.94,1.308304) +- (0,0.007873)
};
\addplot[wcpplot,error bars/.cd,y dir=both,y explicit] coordinates {
(0.00,1.167518) +- (0,0.009385)
(1.00,1.446449) +- (0,0.010422)
(2.00,1.714928) +- (0,0.014366)
};
\addplot[tiltplot,error bars/.cd,y dir=both,y explicit] coordinates {
(0.06,1.168379) +- (0,0.009014)
(1.06,1.317970) +- (0,0.008737)
(2.06,1.191027) +- (0,0.013780)
};
\end{axis}
\end{tikzpicture}\hfill
\begin{tikzpicture}
\begin{axis}[resultaxis,height=4.5cm,title={Covariate and response shift ($a^\star=1$)},ylabel={Mean total length},
  xmin=-.2,xmax=2.2,xtick={0,1,2},xticklabels={{$(1,0)$},{$(1,0.6)$},{$(1,1.2)$}},
  xlabel={Shift parameters $(a^\star,b^\star)$},ymin=1.0,ymax=2.35,ytick={1.0,1.5,2.0}]
\addplot[cpplot,error bars/.cd,y dir=both,y explicit] coordinates {
(-0.06,1.327921) +- (0,0.007968)
(0.94,1.407777) +- (0,0.009236)
(1.94,1.487582) +- (0,0.010042)
};
\addplot[wcpplot,error bars/.cd,y dir=both,y explicit] coordinates {
(0.00,1.485629) +- (0,0.013359)
(1.00,1.854340) +- (0,0.016825)
(2.00,2.189599) +- (0,0.021411)
};
\addplot[tiltplot,error bars/.cd,y dir=both,y explicit] coordinates {
(0.06,1.481677) +- (0,0.012343)
(1.06,1.695758) +- (0,0.014856)
(2.06,1.524788) +- (0,0.023376)
};
\end{axis}
\end{tikzpicture}
\par\smallskip
\ref{adaptlegend}
\caption{Regression at $\eta=1$, showing target coverage (top) and mean total
set length (bottom), with means $\pm$ one SE over $30$ replications.
Axes give true $(a^\star,b^\star)$. The left column varies the response
shift alone, while the right adds a shift toward noisier inputs. The first
points, $(0,0)$ and $(1,0)$, are no-shift and covariate-shift controls.
Both coefficients are estimated in all six cases. Dashed lines mark
$0.90$ coverage, and horizontal offsets separate methods.}

\label{fig:regression-results}
\end{figure}
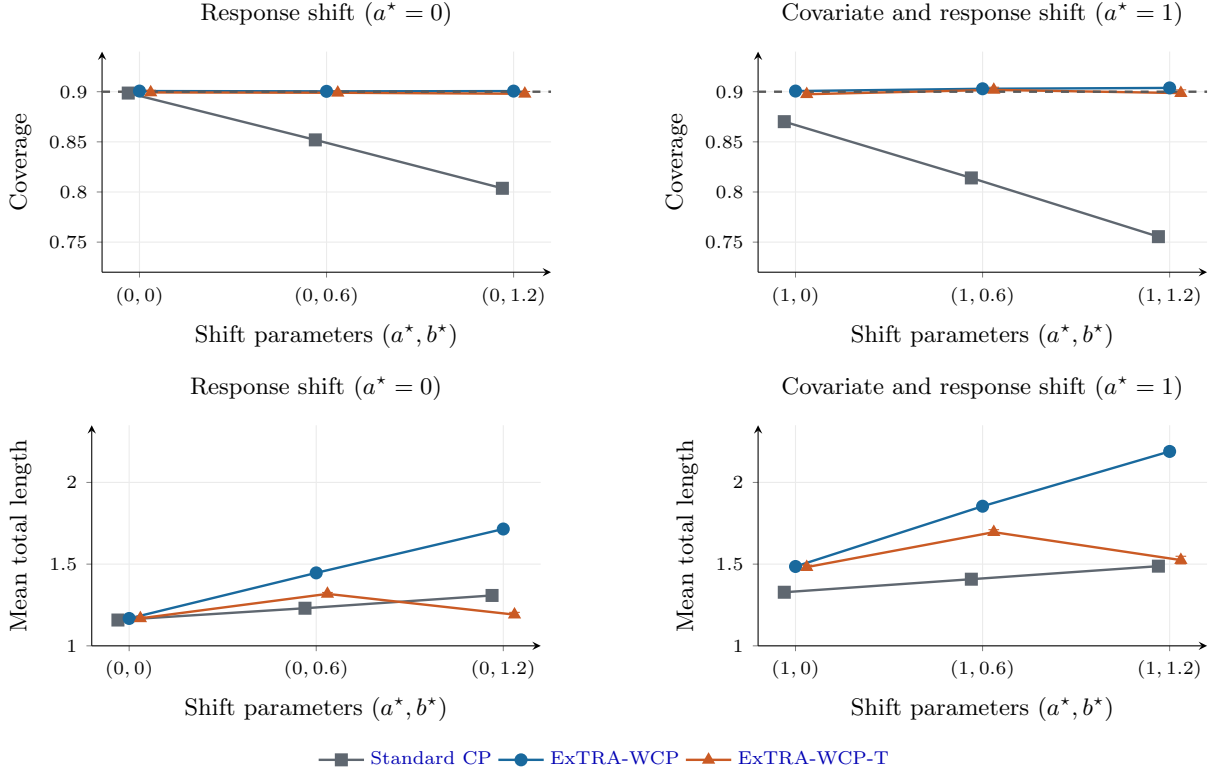

Standard CP loses coverage as either shift increases, while ExTRA-WCP
covers between $0.900$ and $0.904$ (Figure~\ref{fig:regression-results}).
At $(1,1.2)$, input-only WCP covers only $0.815$
(Table~\ref{tab:regression-adaptation}). Its weights
$\widehat M_{\widehat\beta}(x)$ estimate the input ratio up to a common
factor. Even the true input ratio would leave the conditional response
shift uncorrected.

\begin{table}[htbp]
\centering
\caption{Target coverage and mean total prediction-set length under the
combined shift $(a^\star,b^\star)=(1,1.2)$, at nominal coverage $0.90$.
Entries are means (SEs) over $30$ replications. Input-only WCP uses
$\widehat M_{\widehat\beta}(x)$ as an unnormalized input-ratio estimate.
The two-stage baseline
estimates the covariate tilt and mode prior separately. Oracle methods
use the true joint ratio. All methods share the learned source predictor
and evaluation samples.}
\label{tab:regression-adaptation}
\small
\begin{tabular}{lrr}
\hline
Method & Coverage & Total length\\
\hline
Standard CP & 0.755 (0.003) & 1.488 (0.010)\\
Input-only WCP & 0.815 (0.004) & 1.728 (0.015)\\
ExTRA-WCP & 0.904 (0.003) & 2.190 (0.021)\\
ExTRA-WCP-T & 0.899 (0.003) & 1.525 (0.023)\\
Two-stage WCP & 0.902 (0.003) & 2.181 (0.024)\\
Two-stage WCP-T & 0.902 (0.003) & 1.535 (0.019)\\
Oracle WCP & 0.901 (0.003) & 2.174 (0.023)\\
Oracle WCP-T & 0.902 (0.002) & 1.535 (0.017)\\
\hline
\end{tabular}
\end{table}

Tilting gives near-nominal coverage at lengths close to, or shorter
than, standard CP's under-covering sets. At $(1,1.2)$, ExTRA-WCP-T
has coverage $0.899$ and length $1.525$, compared with $0.755$ and
$1.488$ for standard CP. At $(0,1.2)$, tilting gives coverage $0.898$
and length $1.191$, versus $0.804$ and $1.308$ for standard CP.
Relative to ExTRA-WCP at $(1,1.2)$,
the mean paired length reduction is $30.34\%$ ($95\%$ interval
$[28.56\%,32.12\%]$). The coverage difference is $-0.0052$
($[-0.0121,0.0017]$).

The length gain comes from how the score allocates space between the
two modes. The source predictor assigns little probability to the
positive mode. With the source score, weighted calibration must admit
lower-density responses to cover that mode in the target population,
also expanding the region around the mode already common in the source.
Tilting moves predictive mass toward the positive mode, allowing the
set to cover the target with less total length. The oracle length gain
of $29.33\%$ shows that this effect persists without ratio estimation
error. Two-stage gives similar coverage and length, supporting the
benefit of weighting and tilting without establishing an advantage for
ExTRA's ratio estimator.

With $b^\star=0$, length-reduction intervals include zero. In the
covariate-shift control, the true tilt cancels from the predictor, but
fitting an unnecessary response coefficient changes the score. The paired
WCP-T minus WCP coverage difference is $-0.0032$ ($95\%$ interval
$[-0.0050,-0.0013]$).
The larger penalty preserves the main length gain at $(1,1.2)$
($27.56\%$, with WCP and WCP-T coverage $0.898$ and $0.904$).
All evaluated synthetic regression sets have finite length under full inversion.

\subsection{A response shift that depends on the noise input}
\label{subsec:regression-interaction}

The previous experiment changes mode frequencies by the same tilt at
every input. We next ask what happens when this response shift itself
depends on the noise level. Keeping the same source at $\eta=1$, use
\begin{equation}
 h^U_{a,b,c}(u,v,y)=\exp\{au+(b+cu)\operatorname{sign}(y)\},
 \qquad(a^\star,b^\star,c^\star)=(1,0,2).
 \label{eq:interaction-tilt}
\end{equation}
The target mode log odds become $-2+3v+4u$, making positive responses
more likely at noisier inputs. This shift is nonseparable because the
ratio between signs, $e^{2(b+cu)}$, varies with $u$. The nonconstant
source mode probability in $v$ identifies $(a,b,c)$
(Proposition~\ref{prop:separated-identification}), unlike the ambiguous
interaction with $V$ in Proposition~\ref{prop:learned-ambiguity}.

After preliminary results suggested a benefit, fitting choices were
fixed before $30$ fresh replications. We compare full ExTRA with
separable ExTRA ($c=0$), holding the source model and samples fixed.
Only the full fitted family represents the generating shift. Comparing
full and separable ExTRA assesses the value of that representation.
Comparing WCP and WCP-T within either family isolates the additional
effect of changing the score.

\begin{table}[htbp]
\centering
\caption{Response interaction at $(a^\star,b^\star,c^\star)=(1,0,2)$.
Entries are means (SEs) over $30$ replications at nominal coverage $0.90$.
The full family includes $U\operatorname{sign}(Y)$, which separable fits
omit. Oracle weights retain the learned source predictor. TV uses population
normalization. The full and separable ExTRA rows compare the effect
of including the interaction.}
\label{tab:regression-interaction}
\small
\setlength{\tabcolsep}{3pt}
\begin{tabular}{lrrrrr}
\hline
 & \multicolumn{2}{c}{WCP} & \multicolumn{2}{c}{WCP-T} & \\
Ratio estimator & Coverage & Length & Coverage & Length & TV\\
\hline
ExTRA, full & 0.900 (0.004) & 1.760 (0.022) & 0.899 (0.003) & 1.155 (0.012) & 0.041 (0.003)\\
ExTRA, separable & 0.875 (0.003) & 1.621 (0.014) & 0.892 (0.002) & 1.442 (0.010) & 0.384 ($<0.001$)\\
Oracle, full & 0.900 (0.004) & 1.760 (0.022) & 0.899 (0.002) & 1.146 (0.010) & 0\\
\hline
\end{tabular}
\end{table}

Within the full family, tilting reduces mean length by $34.24\%$
($95\%$ interval $[32.94\%,35.53\%]$), with coverage difference
$-0.0014$ ($[-0.0075,0.0047]$). Relative to separable ExTRA-WCP-T,
the full model's tilted sets are $19.85\%$ shorter and cover $0.899$
instead of $0.892$ (Table~\ref{tab:regression-interaction}). Its TV
error is $0.041$ rather than $0.384$, and oracle results are close.
Two-stage behaves similarly to separable ExTRA. Its WCP and WCP-T
coverages are $0.875$ and $0.892$, with lengths $1.617$ and $1.441$.
The result supports representing the response--noise interaction,
without establishing superiority over other estimators that can
represent it.

\subsection{Identification sensitivity and ratio error}
\label{subsec:regression-sensitivity}
\label{subsec:diagnostics}

The strong-signal study leaves open whether target inputs carry enough
information for reliable fitting when source inputs are less predictive
of the response mode. We examine this within the same separable family,
reducing the mode signal to
$\eta\in\{1,0.5,0.25,0\}$ at fixed shift $(a^\star,b^\star)=(1,1.2)$.
Adjusting $d_\eta$ preserves $p_+$ and hence the target mode frequency.
Component distributions, features, fitting, and sample sizes remain
fixed. The $\eta=1$ case reuses the main replications. Each other level
uses $30$ new replications. Neither $\eta$ nor $d_\eta$ is supplied
to the estimator. For every $\eta>0$, the ratio is identified, whereas
at $\eta=0$,
\begin{equation}
 M_{a,b}(u,v)=e^{au}\{p_+e^b+(1-p_+)e^{-b}\},\qquad
 \frac{dQ_{a,b,X}}{dP_X}(u,v)=\frac{e^{au}}{\E_P e^{aU}},
 \label{eq:hidden-mode-marginal}
\end{equation}
so target inputs contain no information about $b$.

This ladder also changes the source score's alignment with target
responses. At $\eta=1$, positive target responses tend to occur where
the source already assigns that mode higher probability. At $\eta=0$,
the population score carries the same positive-mode penalty $-\log p_+$
everywhere. Standard CP's falling coverage therefore does not indicate
a larger mode-prior shift.

\begin{figure}[htbp]
\centering
\begin{tikzpicture}
\begin{axis}[resultaxis,title={Target coverage},ylabel={Coverage},xmin=-.25,xmax=3.25,
  xtick={0,1,2,3},xticklabels={1,0.5,0.25,0},xlabel={Mode signal $\eta$},ymin=.40,ymax=.96,ytick={.5,.6,.7,.8,.9},legend columns=3,legend to name=sensitivitylegend]
\addplot[cpplot,error bars/.cd,y dir=both,y explicit] coordinates {
(-0.06,0.755450) +- (0,0.003137)
(0.94,0.724567) +- (0,0.004154)
(1.94,0.703083) +- (0,0.005037)
(2.94,0.686150) +- (0,0.006413)
};
\addlegendentry{Standard CP}
\addplot[wcpplot,error bars/.cd,y dir=both,y explicit] coordinates {
(0.00,0.903783) +- (0,0.003032)
(1.00,0.904033) +- (0,0.003589)
(2.00,0.818792) +- (0,0.024065)
(3.00,0.727017) +- (0,0.039488)
};
\addlegendentry{ExTRA-WCP}
\addplot[tiltplot,error bars/.cd,y dir=both,y explicit] coordinates {
(0.06,0.898608) +- (0,0.003221)
(1.06,0.878375) +- (0,0.008784)
(2.06,0.730383) +- (0,0.038636)
(3.06,0.492967) +- (0,0.044682)
};
\addlegendentry{ExTRA-WCP-T}
\addplot[black!65,dashed,no marks,forget plot] coordinates {(-.25,.9)(3.25,.9)};
\end{axis}
\end{tikzpicture}\hfill
\begin{tikzpicture}
\begin{axis}[resultaxis,title={Prediction-set length},ylabel={Mean total length},xmin=-.25,xmax=3.25,
  xtick={0,1,2,3},xticklabels={1,0.5,0.25,0},xlabel={Mode signal $\eta$},ymin=.8,ymax=2.6,ytick={1.0,1.5,2.0,2.5}]
\addplot[cpplot,error bars/.cd,y dir=both,y explicit] coordinates {
(-0.06,1.487582) +- (0,0.010042)
(0.94,1.612170) +- (0,0.009678)
(1.94,1.622997) +- (0,0.009920)
(2.94,1.578865) +- (0,0.011703)
};
\addplot[wcpplot,error bars/.cd,y dir=both,y explicit] coordinates {
(0.00,2.189599) +- (0,0.021411)
(1.00,2.301977) +- (0,0.016805)
(2.00,2.028145) +- (0,0.067988)
(3.00,1.838141) +- (0,0.106367)
};
\addplot[tiltplot,error bars/.cd,y dir=both,y explicit] coordinates {
(0.06,1.524788) +- (0,0.023376)
(1.06,1.673176) +- (0,0.039612)
(2.06,1.541521) +- (0,0.064915)
(3.06,1.103244) +- (0,0.041301)
};
\end{axis}
\end{tikzpicture}
\par\smallskip\ref{sensitivitylegend}
\caption{Sensitivity at fixed shift $(a^\star,b^\star)=(1,1.2)$.
Reducing $\eta$ weakens mode information in the inputs while preserving
mode frequencies and component distributions. It also changes the source
score's alignment with target responses. ExTRA estimation deteriorates
before nonidentification at $\eta=0$. Two-stage results in the text show
that some loss is estimator-specific. Points show means $\pm$ one SE
over $30$ replications. The $\eta=1$ results also appear in
Figure~\ref{fig:regression-results}.
The dashed line marks $0.90$ coverage.}

\label{fig:identification-results}
\end{figure}
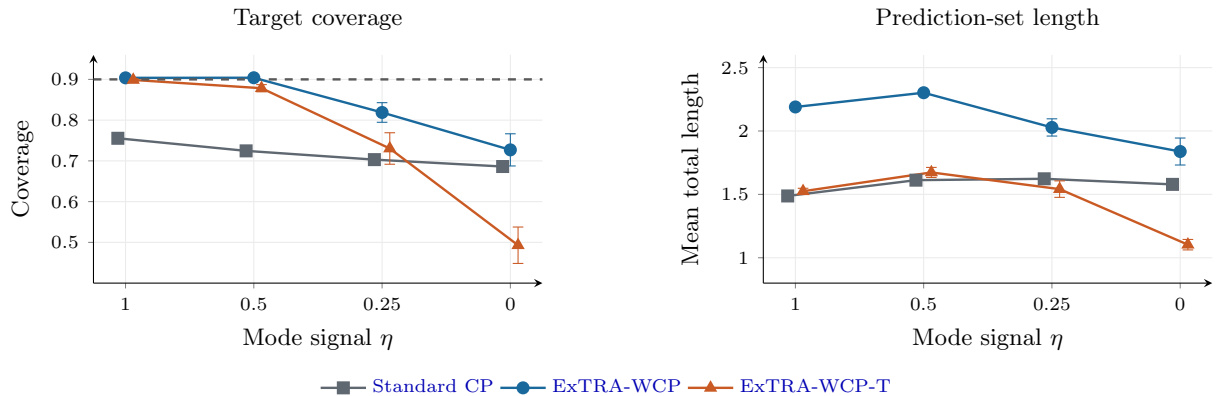

At $\eta=0.25$, ExTRA-WCP and WCP-T cover $0.819$ and $0.730$.
At zero signal they cover $0.727$ and $0.493$
(Figure~\ref{fig:identification-results}). Oracle coverage stays between
$0.898$ and $0.905$. The intermediate case $\eta=0.5$ is noisy.
WCP and WCP-T cover $0.904$ and $0.878$, with paired difference $-0.026$
($95\%$ interval $[-0.049,-0.003]$) and WCP-T replication SD $0.048$.
The larger penalty reverses the ordering to $0.890$ and $0.905$.

Two-stage is more accurate at weak positive signal. At $\eta=0.25$,
its WCP and WCP-T cover $0.904$ and $0.874$, improving over ExTRA by
$0.085$ and $0.144$, with intervals $[0.035,0.135]$ and $[0.060,0.228]$.
At $\eta=0$, their coverages fall to $0.705$ and $0.588$.
Thus the positive-signal losses depend on fitting accuracy, while zero
signal prevents identification itself. The alternative estimator is
useful here because it shows that weak source-mode information can
remain usable even when ExTRA's fitted ratio deteriorates.

For each synthetic fit, we compute
$d_r=\frac12\E_P|w^\star-\widehat w_r|$ using the known source
distribution and population normalization
(Appendix~\ref{app:ratio-diagnostics}). The shared ratio gives WCP
and WCP-T the same bound $0.90-d_r$. Table~\ref{tab:ratio-diagnostics}
reports these errors and coverage differences. The averaged bound is
not a confidence bound for empirical coverage.

\begin{table}[ht!]
\centering
\caption{Joint-ratio error and paired coverage differences for synthetic
classification and separable regression over $30$ replications. TV entries
are means (SEs) for each estimator. The bound and last two columns refer
to ExTRA and report $0.90-\overline d$, WCP-T minus WCP coverage,
and its paired $95\%$ interval, respectively. The two-stage baseline is not used for classification.}
\label{tab:ratio-diagnostics}
\small
\setlength{\tabcolsep}{4pt}
\begin{tabular}{lrrrrr}
\hline
Setting & ExTRA TV & Two-stage TV & ExTRA bound & Difference & $95\%$ interval\\
\hline
Classification & 0.187 (0.028) & --- & 0.713 & $-0.110$ & $[-0.164,-0.057]$\\
\multicolumn{6}{l}{Main regression ($\eta=1$), $(a^\star,b^\star)$}\\
$(0,0)$ & 0.039 (0.005) & 0.014 (0.001) & 0.861 & $-0.001$ & $[-0.003,0.000]$\\
$(0,0.6)$ & 0.045 (0.005) & 0.020 (0.002) & 0.855 & $-0.001$ & $[-0.006,0.003]$\\
$(0,1.2)$ & 0.040 (0.004) & 0.023 (0.003) & 0.860 & $-0.003$ & $[-0.007,0.002]$\\
$(1,0)$ & 0.046 (0.006) & 0.019 (0.002) & 0.854 & $-0.003$ & $[-0.005,-0.001]$\\
$(1,0.6)$ & 0.050 (0.006) & 0.019 (0.002) & 0.850 & $-0.001$ & $[-0.006,0.004]$\\
$(1,1.2)$ & 0.046 (0.005) & 0.024 (0.003) & 0.854 & $-0.005$ & $[-0.012,0.002]$\\
\multicolumn{6}{l}{Mode-signal sensitivity, $(a^\star,b^\star)=(1,1.2)$}\\
$\eta=0.5$ & 0.106 (0.013) & 0.048 (0.006) & 0.794 & $-0.026$ & $[-0.049,-0.003]$\\
$\eta=0.25$ & 0.315 (0.042) & 0.082 (0.011) & 0.585 & $-0.088$ & $[-0.134,-0.043]$\\
$\eta=0$ & 0.470 (0.048) & 0.529 (0.038) & 0.430 & $-0.234$ & $[-0.269,-0.199]$\\
\hline
\end{tabular}
\end{table}

Two-stage has lower mean TV at every positive signal level studied.
At $\eta=0.25$, its TV is $0.082$ versus ExTRA's $0.315$, confirming
that usable information remains. Yet TV alone does not order scores.
In classification the common error is $0.187$, while WCP exceeds
nominal coverage by $0.024$ and WCP-T falls short by $0.086$.
Measuring the signed transfer term would additionally require
miscoverage under the surrogate in \eqref{eq:score-transfer-identity}.

Fitted coefficients reveal instability. At $\eta=1$, $\widehat b$
averages $1.245$ (SD $0.139$), about $1.8$ SEs above $1.2$.
Gate and normalization errors may contribute. At $\eta=0.25$ and $0$,
$8/30$ and $15/30$ estimates are negative. Wrong-sign and excessively
positive fits accompany coverage losses. At zero signal, six estimates
reach $b=3$. The unpenalized population criterion is flat in $b$ and
ridge selects zero, but noise in the learned gate and separate empirical
normalizer produces variation for the optimizer to fit. All runs are
retained. These post hoc associations do not yield a score-selection
rule without target responses.

\subsection{Classification and regression on real data}
\label{subsec:real-data}

Gas Sensor Array Drift at Different Concentrations
(\href{https://doi.org/10.24432/C5MK6M}{UCI 270},
\citealp{VergaraA2012sensors,RodriguezLujanI2014chemolab}) has $13910$
observations and six classes. Batches 1--7 form the source and batches
8--10 are separate targets. Communities and Crime
(\href{https://doi.org/10.24432/C53W3X}{UCI 183}, \citealp{RedmondM2002ejor}) has $1994$ communities
and a response on $[0,1]$.
Each Census region serves as the target in turn, with the other three regions forming the source.

In both studies, we use extremely randomized trees \citep{GeurtsP2006mlj},
with APS scores for Gas and a truncated-normal predictive density for
Crime. Neither predictor nor tilt is assumed correctly specified.
Target folds separate fitting from evaluation. We use eight randomized
source partitions, aggregating target folds by test count within each
partition. SEs describe variation across partitions of fixed datasets,
rather than newly sampled regions or sensor systems.

\begin{figure}[ht!]
\centering
\begin{tikzpicture}
\begin{axis}[resultaxis,title={Gas coverage},ylabel={Coverage},
 xmin=7.7,xmax=10.3,xtick={8,9,10},xlabel={Target batch},
 ymin=.70,ymax=1.00,ytick={.70,.80,.90,1.00},
 legend columns=3,legend to name=reallegend]
\addplot[cpplot,error bars/.cd,y dir=both,y explicit] coordinates {(7.94,0.968112) +- (0,0.001922) (8.94,0.922074) +- (0,0.009395) (9.94,0.923056) +- (0,0.001283)};
\addlegendentry{Standard CP}
\addplot[wcpplot,error bars/.cd,y dir=both,y explicit] coordinates {(8.00,0.966837) +- (0,0.002007) (9.00,0.910638) +- (0,0.010116) (10.00,0.913368) +- (0,0.001580)};
\addlegendentry{ExTRA-WCP}
\addplot[tiltplot,error bars/.cd,y dir=both,y explicit] coordinates {(8.06,0.930697) +- (0,0.009998) (9.06,0.739628) +- (0,0.009282) (10.06,0.829826) +- (0,0.002777)};
\addlegendentry{ExTRA-WCP-T}
\addplot[black!65,dashed,no marks,forget plot] coordinates {(7.7,.9)(10.3,.9)};
\end{axis}
\end{tikzpicture}\hfill
\begin{tikzpicture}
\begin{axis}[resultaxis,title={Gas set size},ylabel={Mean set size},
 xmin=7.7,xmax=10.3,xtick={8,9,10},xlabel={Target batch},ymin=1.6,ymax=2.8]
\addplot[cpplot,error bars/.cd,y dir=both,y explicit] coordinates {(7.94,2.127976) +- (0,0.025184) (8.94,2.414894) +- (0,0.021590) (9.94,2.667743) +- (0,0.010579)};
\addplot[wcpplot,error bars/.cd,y dir=both,y explicit] coordinates {(8.00,2.054422) +- (0,0.014068) (9.00,2.323404) +- (0,0.023511) (10.00,2.578507) +- (0,0.012482)};
\addplot[tiltplot,error bars/.cd,y dir=both,y explicit] coordinates {(8.06,1.786990) +- (0,0.028933) (9.06,1.870479) +- (0,0.032493) (10.06,2.332361) +- (0,0.020397)};
\end{axis}
\end{tikzpicture}
\par\smallskip
\begin{tikzpicture}
\begin{axis}[resultaxis,title={Crime coverage},ylabel={Coverage},
 xmin=-.3,xmax=3.3,xtick={0,1,2,3},
 xticklabels={Northeast,Midwest,South,West},
 ymin=.70,ymax=1.00,ytick={.70,.80,.90,1.00}]
\addplot[cpplot,error bars/.cd,y dir=both,y explicit] coordinates {(-0.06,0.955763) +- (0,0.005186) (0.94,0.937911) +- (0,0.005105) (1.94,0.721955) +- (0,0.008897) (2.94,0.888333) +- (0,0.003164)};
\addplot[wcpplot,error bars/.cd,y dir=both,y explicit] coordinates {(0.00,0.964692) +- (0,0.005596) (1.00,0.952303) +- (0,0.005832) (2.00,0.802684) +- (0,0.015533) (3.00,0.903056) +- (0,0.005893)};
\addplot[tiltplot,error bars/.cd,y dir=both,y explicit] coordinates {(0.06,0.909700) +- (0,0.013944) (1.06,0.916530) +- (0,0.006307) (2.06,0.782252) +- (0,0.014567) (3.06,0.900556) +- (0,0.004518)};
\addplot[black!65,dashed,no marks,forget plot] coordinates {(-.3,.9)(3.3,.9)};
\end{axis}
\end{tikzpicture}\hfill
\begin{tikzpicture}
\begin{axis}[resultaxis,title={Crime set length},ylabel={Mean total length},
 xmin=-.3,xmax=3.3,xtick={0,1,2,3},
 xticklabels={Northeast,Midwest,South,West},ymin=.32,ymax=.48]
\addplot[cpplot,error bars/.cd,y dir=both,y explicit] coordinates {(-0.06,0.420461) +- (0,0.008865) (0.94,0.407180) +- (0,0.006248) (1.94,0.346223) +- (0,0.008689) (2.94,0.426121) +- (0,0.005515)};
\addplot[wcpplot,error bars/.cd,y dir=both,y explicit] coordinates {(0.00,0.451294) +- (0,0.012606) (1.00,0.434771) +- (0,0.010302) (2.00,0.432637) +- (0,0.019389) (3.00,0.452318) +- (0,0.010277)};
\addplot[tiltplot,error bars/.cd,y dir=both,y explicit] coordinates {(0.06,0.453574) +- (0,0.009178) (1.06,0.442974) +- (0,0.009723) (2.06,0.414726) +- (0,0.015207) (3.06,0.440232) +- (0,0.005972)};
\end{axis}
\end{tikzpicture}
\par\smallskip\ref{reallegend}
\caption{Gas classification and Crime regression at nominal coverage
$0.90$ (dashed). Points show means $\pm$ one SE across eight source
partitions of each fixed dataset. Horizontal offsets separate methods,
and lines connect discrete target groups. Tilting gives smaller Gas
sets but substantial undercoverage in batches 9 and 10. Crime shows
no consistent length reduction, and all methods under-cover in South.}
\label{fig:real-data}
\end{figure}
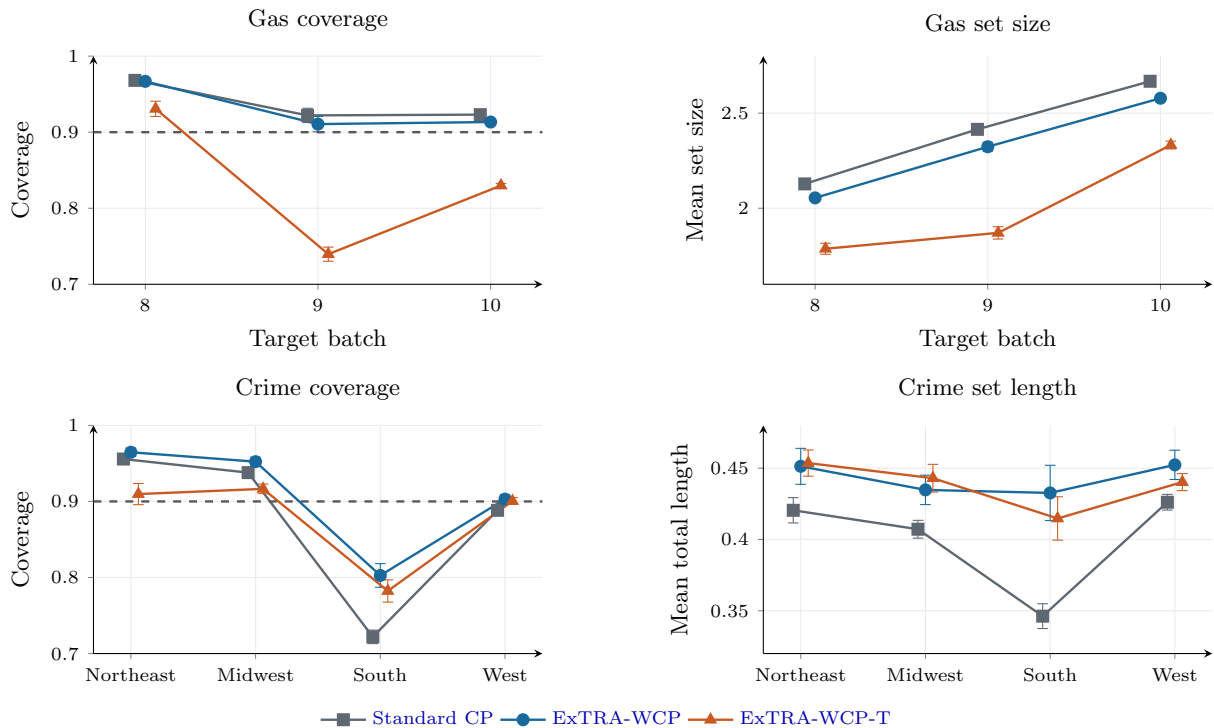

On Gas, tilting reduces set size in every batch, but lowers coverage
from $0.911$ to $0.740$ in batch 9 and from $0.913$ to $0.830$ in batch 10
(Figure~\ref{fig:real-data}). On Crime, weighting expands sets and
improves coverage in West and South, although all methods under-cover
in South. Tilting lowers coverage without shortening sets in Northeast
and Midwest. West's paired length reduction is $2.52\%$ at similar
coverage, with interval $[-0.04\%,5.08\%]$. Neither dataset shows
a consistent benefit from tilting.

\section{Discussion and conclusion}
\label{sec:discussion}

Weighting calibration and reshaping predictions have different effects.
In synthetic regression with correctly specified models and informative
target inputs, tilting gives substantially shorter sets at near-nominal
coverage. The separable study shows this with either ratio estimator.
The interaction study shows why representing the shift also matters. Classification, weak-signal regression, and real data show
that the same adjustment can lose coverage. The common TV bound does
not order scores. Actual coverage depends on where ratio errors and
score misses coincide, as well as surplus coverage under the surrogate target.

Target inputs can identify a ratio only within an assumed family.
Our calculations distinguish an identifiable response--noise interaction
from an indistinguishable interaction with the source mode signal.
A good input fit cannot validate restrictions excluding the latter.
Choosing tilting reliably without target responses remains open.
Robust conformal methods protect ranges of targets
\citep{CauchoisM2024jasa,AiJ2024icml}, while PAC prediction sets account
for estimated weights under covariate or label shift
\citep{ParkSD2022iclr,SiW2024iclr}. Extending such protection to ExTRA
requires uncertainty sets that cover estimation error and unresolved
shifts, whose extent target inputs alone generally cannot determine.

\bibliographystyle{abbrvnat}
\bibliography{sjc}

\clearpage
\appendix
\section{Identification in the regression experiments}
\label{app:identification}

The following calculations explain which response shifts target
inputs can distinguish in our regression experiments. The response
enters the tilt through $Z=\operatorname{sign}(Y)$, whose source
probability depends on $V$. We first allow the response shift to
interact with $U$, then show how an interaction with $V$ creates
an exact ambiguity. The distinction concerns the chosen shift
family, even when the source distribution is known.

\medskip
\begin{proposition}[Identification with a response interaction in $U$]
\label{prop:separated-identification}
Fix a source distribution $P$. Suppose $(U,V)$ is supported on
$[-1,1]^2$ with a density positive almost everywhere on $(-1,1)^2$,
$P(Y=0)=0$, and
\[
 P(Y>0\mid U=u,V=v)=\pi(v),
\]
where $0<\pi(v)<1$ and $\pi$ is not almost everywhere constant.
Within the family
\[
 h^U_{a,b,c}(u,v,y)
 =\exp\{au+(b+cu)\operatorname{sign}(y)\},
\]
the target input marginal uniquely determines $(a,b,c)$.
This includes the separable family obtained by setting $c=0$.
\end{proposition}

\begin{proof}
Write $t(u)=b+cu$. Averaging over the two response modes gives
\[
 M^U_{a,b,c}(u,v)
 =e^{au}\{[1-\pi(v)]e^{-t(u)}+\pi(v)e^{t(u)}\}.
\]
Suppose two parameter triples induce the same normalized input
marginal. Their conditional moments then satisfy
\[
 M^U_{a,b,c}(u,v)
 =C M^U_{a',b',c'}(u,v)
\]
almost everywhere, for some constant $C>0$.
The positive input density and Fubini's theorem imply that,
for almost every $u$, this equality holds for almost every $v$.
Since $1$ and $\pi(v)$ are linearly independent, matching their
coefficients gives
\[
 e^{au-t(u)}=C e^{a'u-t'(u)},
 \qquad
 e^{au+t(u)}=C e^{a'u+t'(u)},
\]
where $t'(u)=b'+c'u$.
Dividing the second equality by the first yields
$e^{2t(u)}=e^{2t'(u)}$, hence $t(u)=t'(u)$ almost everywhere.
Therefore $b=b'$ and $c=c'$.
Substitution gives $e^{(a-a')u}=C$ almost everywhere, which
requires $a=a'$ and $C=1$.
All normalizers are finite and positive because the inputs
and the response sign are bounded.
\end{proof}

For the source model in \eqref{eq:adaptation-source},
$\pi_\eta(v)=\operatorname{logit}^{-1}(d_\eta+3\eta v)$
is nonconstant whenever $\eta>0$, so the proposition applies.
As $\eta\downarrow0$, the intercept adjustment preserving $p_+$
gives
\[
 d_\eta\to\operatorname{logit}(p_+),
 \qquad
 \pi_\eta(v)\to p_+
\]
uniformly in $v$. In the separable family, $c=0$, the normalized
input ratio consequently converges to
\[
 \frac{e^{au}}{\E_P e^{aU}},
\]
which does not depend on $b$.
At $\eta=0$, target inputs therefore contain no information
about the separable response-shift coefficient, as shown in
\eqref{eq:hidden-mode-marginal}. This limit explains why the
sensitivity study varies $\eta$.

A different ambiguity arises when the response shift interacts
with $V$, the input that controls source mode probabilities.

\medskip
\begin{proposition}[Ambiguity with a response interaction in $V$]
\label{prop:learned-ambiguity}
Under the source model in \eqref{eq:adaptation-source}, consider
\[
 h^V_{a,b,c}(u,v,y)
 =\exp\{au+(b+cv)\operatorname{sign}(y)\}.
\]
The parameter triples
\[
 (a,b,c)
 \quad\text{and}\quad
 (a,-d_\eta-b,-3\eta-c)
\]
induce the same target input marginal. Their joint ratios differ
unless $b=-d_\eta/2$ and $c=-3\eta/2$.
At each input, their conditional positive-mode probabilities
sum to one.
\end{proposition}

\begin{proof}
Averaging over the source response modes gives
\[
 \begin{aligned}
 M^V_{a,b,c}(u,v)
 &=e^{au}\{[1-\pi_\eta(v)]e^{-(b+cv)}
                  +\pi_\eta(v)e^{b+cv}\}\\
 &=e^{au}
 \frac{\cosh\{d_\eta/2+b+(3\eta/2+c)v\}}
      {\cosh\{d_\eta/2+(3\eta/2)v\}}.
 \end{aligned}
\]
The stated parameter transformation negates the argument of
the numerator's hyperbolic cosine. Since $\cosh$ is even,
the conditional moment and its source expectation remain
unchanged. The normalized input marginals therefore coincide.

Because the normalizers agree, equality of the joint ratios
would require equality of the tilt functions. Both response
signs have positive conditional probability, so this requires
\[
 2b+d_\eta+(2c+3\eta)v=0
\]
almost everywhere. This occurs only when
$b=-d_\eta/2$ and $c=-3\eta/2$.

Finally, the target positive-mode probability is
\[
 Q^V_{a,b,c}(Z=+1\mid u,v)
 =\operatorname{logit}^{-1}
   \{d_\eta+2b+(3\eta+2c)v\}.
\]
The transformation negates this logit, so the two probabilities
are complementary.
\end{proof}

For a concrete example, take $\eta=1$, for which $d_1=-2$.
Within the $V$-interaction family, the triples $(1,1.2,0)$
and $(1,0.8,-3)$ give identical input marginals but different
joint ratios. The separable and $U$-interaction families used
in the experiments both exclude the second shift by omitting
$V\operatorname{sign}(Y)$.

Thus, the identification result depends on this feature
restriction. Enlarging the family to allow both interactions
introduces indistinguishable joint shifts. Matching the target
input distribution cannot establish that the excluded
interaction is absent.

\section{Essential experimental details}
\label{app:implementation}

\paragraph{Synthetic fitting.}
ExTRA uses the ridge penalty $\lambda\|\beta\|_2^2/2$ in
\eqref{eq:extra-objective} and L-BFGS-B with coefficient bounds.
Classification and real-data fits start from zero. Synthetic
regression uses the multiple starts specified below.

The classification predictor is multinomial logistic regression
with inverse regularization strength $C=100$.
The generating tilt has class slopes $(0.70,-0.15)$,
$(-0.40,0.55)$, $(0.20,-0.60)$ and intercepts $(0.40,-0.30,0)$.
ExTRA uses $\lambda=0.01$, slopes in $[-1.5,1.5]$, and the first
two intercepts in $[-3.5,3.5]$, with the last fixed at zero.
The target component means are $(-1.3,-0.15)$, $(1.6,0.55)$,
and $(0.2,1.9)$. The slope bounds exclude alternative component
assignments, so Gaussian mixture identification up to permutation
identifies the joint ratio at this target.

Regression uses logistic mode probabilities on $(u,v)$ with
$C=100$, per-sign Gaussian locations linear in $(1,u,v)$, and
a shared log scale linear in $(1,u)$.
The six location and two scale coefficients are fitted jointly
by truncated-normal maximum likelihood.
Bounds are $[-3,3]$ for location intercepts, $[-1,1]$ for location
slopes, $[-4,0]$ for the log-scale intercept, and $[-1.5,1.5]$
for its slope.
Separable ExTRA uses $a\in[-1.5,1.5]$, $b\in[-3,3]$, and starts
$(0,0)$ and $(0,\pm1.5)$, retaining the converged solution with
the largest objective.
The main ridge coefficient is $1/n_{\mathrm{shift}}$.
The regularization check uses
$\sqrt{1/n_{\mathrm{shift}}+1/m}$.
Both were fixed after development pilots and before final evaluation.
All replications yield a converged fit. Six fits with the smaller
penalty at $\eta=0$ reach $b=3$.
Convergence does not certify a global maximum.
Input-only WCP retains the source score and uses
$\widehat M_{\widehat\beta}(x)$ for calibration and candidate weights.

The interaction experiment adds $c\in[-3,3]$ and uses starts
$(0,0,0)$, $(0,\pm1.5,0)$, and $(0,0,\pm1.5)$.
All starts converge, with no selected fit at a coefficient bound
and no discarded replications.
Target inputs are sampled by rejection using $M^U_{a,b,c}(u,v)$.
Conditional positive-mode probabilities are
$\operatorname{logit}^{-1}\{-2+3v+2(b+cu)\}$, with unchanged
within-mode response distributions.
Source fitting, sample sizes, and penalties match the separable
experiment, with independent random seeds.
The unchanged two-stage estimator is misspecified here.

For identification sensitivity, solving
$\E_V\operatorname{logit}^{-1}(d_\eta+3\eta V)=p_+$
gives $d_\eta\simeq-2.000,-1.481,-1.331,-1.279$ at
$\eta=1,0.5,0.25,0$, respectively.
These values generate data only. The fitted logistic model
estimates its intercept and both slopes.
Fitting choices remain fixed.
The $\eta=1$ results reuse the separable setting
$(a^\star,b^\star)=(1,1.2)$. Other levels use independent seeds.

\paragraph{Response-set inversion.}
\label{app:inversion}
Synthetic regression sets are inverted analytically over both
complete response half-lines.
Conditional on calibration and the fitted functions, write
$V_i=S(X_i,Y_i)$, $W=\sum_iH_i$, and let $H_z$ be the candidate
weight on $zy>0$.
Define
\begin{equation}
 q_{S,z}(x)
 =\inf\left\{
 t:\sum_iH_i\ind\{V_i\leq t\}
       \geq(1-\alpha)(W+H_z)
 \right\},
 \label{eq:mode-threshold}
\end{equation}
with $\inf\varnothing=+\infty$.
On this half-line, the weighted-rank rule is equivalent to
$S(x,y)\leq q_{S,z}(x)$.

For
\[
 S(x,y)=c_{S,z}(x)
       +\frac{(y-\widehat\mu_z(x))^2}
              {2\widehat\sigma_z^2(x)},
\]
a finite threshold $q_{S,z}\geq c_{S,z}$ gives
\[
 \begin{aligned}
 r_{S,z}(x)
 &=\widehat\sigma_z(x)\sqrt{2(q_{S,z}(x)-c_{S,z}(x))},\\
 I_{S,z}(x)
 &=[\widehat\mu_z(x)-r_{S,z}(x),
    \widehat\mu_z(x)+r_{S,z}(x)]\cap\{y:zy>0\}.
 \end{aligned}
\]
The region is empty if $q_{S,z}<c_{S,z}$ and is the entire
half-line if $q_{S,z}=+\infty$.
For a nonempty intersection, let $L_{S,z}$ and $U_{S,z}$ be
its endpoints. If $F_z(\cdot\mid x)$ is the true source
within-mode CDF, then
\begin{equation}
 m_{S,z}(x\mid\Dcal)
 =1-\{F_z(U_{S,z}(x)\mid x)-F_z(L_{S,z}(x)\mid x)\}.
 \label{eq:mode-miss}
\end{equation}
An empty interval has included mass zero.
Endpoint conventions are immaterial under the continuous model.
Averaging over calibration gives $m_{S,z}(x)$.
No tails are trimmed, all evaluated lengths are finite,
and no infinite lengths are discarded.

\paragraph{Two-stage estimation.}
\label{app:two-stage}
The separable regression shift in \eqref{eq:adaptation-tilt}
has a factor $e^{au}$ that reweights the input $U$ and a factor
$e^{bz}$ that changes the frequency of the response mode
$Z=\operatorname{sign}(Y)$. The two-stage baseline estimates
these changes separately. We include it to assess whether the
benefits of weighting and predictive tilting depend on ExTRA's
particular ratio estimator. Its fitted tilt is used in the same
WCP and WCP-T procedures, with the same source predictor,
calibration sample, and evaluation data.

\emph{First, estimate the change in $U$.}
Because $U$ and $Z$ are independent under the source,
the target marginal of $U$ depends on $a$ alone.
We estimate $a$ by matching this marginal, minimizing
\[
 \log\left\{
 n_{\mathrm{shift}}^{-1}\sum_i e^{aU_i^P}
 \right\}
 -a\overline U_Q+\frac{a^2}{2n_{\mathrm{shift}}}
\]
over $[-1.5,1.5]$, where $\overline U_Q$ is the target sample
mean of $U$. The last term is a ridge penalty.
The weights $r_i=e^{\widehat aU_i^P}$ then adjust the source
shift sample for this estimated covariate change.

\emph{Next, estimate the frequency of positive responses.}
After the true covariate adjustment, the remaining factor
$e^{bz}$ changes only the mode frequencies, preserving the
input distribution within each mode. This is label shift on
the binary variable $Z$, whose source labels are observed
as response signs. Soft black-box shift estimation (BBSE)
\citep{LiptonZ2018icml} uses this structure to infer the target
positive-mode probability from target inputs alone.

Let $g(x)=\widehat P(Z=+1\mid x)$ be the shared fitted logistic
predictor and $z_i=\operatorname{sign}(Y_i^P)$. Compute
\[
 \widehat p_a=\frac{\sum_i r_i\ind\{z_i=+1\}}{\sum_i r_i},
 \qquad
 \widehat c_z=\frac{\sum_{i:z_i=z}r_i g(X_i^P)}
                   {\sum_{i:z_i=z}r_i},
 \qquad
 \overline g_Q=\frac1m\sum_jg(X_j^Q).
\]
Here $\widehat p_a$ is the positive-mode frequency after
covariate reweighting, $\widehat c_z$ is the corresponding
mean predictor output within mode $z$, and $\overline g_Q$
is its mean on target inputs. If the target positive-mode
probability is $q_+$, its mean output is a mixture of the two
within-mode means. BBSE therefore matches
\[
 \overline g_Q
 =(1-q_+)\widehat c_-+q_+\widehat c_+,
 \qquad
 \widetilde q_+
 =\frac{\overline g_Q-\widehat c_-}
        {\widehat c_+-\widehat c_-}.
\]
Thus, a change in the average predictor output reveals a
change in mode frequency whenever the two within-mode means
are distinct. Target responses are not needed.

The response tilt multiplies positive-mode mass by $e^b$ and
negative-mode mass by $e^{-b}$, increasing the log odds by
$2b$. To impose the same bound $|b|\leq3$ as ExTRA, we clip
$\widetilde q_+$ to $[q_{\min},q_{\max}]$, where
\[
 q_{\min}=\operatorname{logit}^{-1}
          \{\operatorname{logit}(\widehat p_a)-6\},
 \qquad
 q_{\max}=\operatorname{logit}^{-1}
          \{\operatorname{logit}(\widehat p_a)+6\},
\]
and convert the clipped estimate $\widehat q_+$ to
\[
 \widehat b
 =\{\operatorname{logit}(\widehat q_+)
    -\operatorname{logit}(\widehat p_a)\}/2.
\]
The resulting tilt is $e^{\widehat a u+\widehat b z}$.
A gap $|\widehat c_+-\widehat c_-|<10^{-12}$ would select
$\widehat b=0$. No fitted gap is this small.
Clipping occurs in $2/30$ sensitivity fits at $\eta=0.25$
and $13/30$ at $\eta=0$, and nowhere else.
All fits are retained.

At the population level, with the true covariate adjustment,
the mixture equation holds for any fixed predictor $g$.
Estimating the mode frequency requires distinct within-mode means. At $\eta=0$, those means coincide because the inputs
carry no information about $Z$. The construction estimates
a binary mode prior within this regression model, rather
than applying BBSE directly to a continuous response.
The same baseline is also evaluated in the interaction study,
where its separable tilt cannot represent the input-dependent
response coefficient $b+cu$.

\paragraph{Population ratio diagnostics.}
\label{app:ratio-diagnostics}
TV diagnostics normalize each fitted tilt by its true source
expectation $\E_P h_{\widehat\beta}$, using neither the learned
conditional distribution nor the empirical fitting normalizer.
For separable regression, $U$ and $Z$ are source-independent.
Writing $p_z=P(Z=z)$, the normalizer and replication-specific
TV error are
\[
 Z(a,b)=\frac{\sinh a}{a}(p_+e^b+p_-e^{-b}),
\]
\[
 d_r=\frac14\sum_{z=\pm1}p_z\int_{-1}^1
 \left|
 \frac{e^{a^\star u+b^\star z}}{Z(a^\star,b^\star)}
 -
 \frac{e^{\widehat a_r u+\widehat b_r z}}
      {Z(\widehat a_r,\widehat b_r)}
 \right|\,du,
\]
with $\sinh(a)/a=1$ at $a=0$.
Quadrature splits at crossings and uses absolute and relative
tolerances of $10^{-11}$.
The formula applies at every $\eta$ because $p_+$ is held fixed.

For the interaction family, replace each mode's slope $a$
by $a+cz$. Its normalizer is
\[
 Z(a,b,c)=\sum_{z=\pm1}p_z e^{bz}
                   \frac{\sinh(a+cz)}{a+cz},
\]
with the ratio interpreted as one when $a+cz=0$.
The same one-dimensional quadrature evaluates full and
separable fits against the true interaction ratio.

For classification, tilting maps source component
$\mathcal N(\mu_k,I)$ to $\mathcal N(\mu_k+\theta_k,I)$,
with class mass proportional to
$\pi_k\exp\{\alpha_k+\theta_k^\top\mu_k+\|\theta_k\|^2/2\}$.
Here $\pi_k=P(Y=k)$ and $\alpha_k$ is the class intercept.
Integrating the positive difference between true and fitted
joint densities within each class reduces to normal CDF
evaluations. Summing over classes gives TV analytically.

\paragraph{Real-data fitting.}
Source partitions allocate $40\%$ to training, $30\%$ to
shift estimation, and $30\%$ to calibration.
Gas partitions are stratified by class.
Its ExtraTrees classifier uses $320$ trees, minimum leaf
size three, square root feature subsampling, equal observation
weights, and no class reweighting.
Tilt features are an intercept and five source principal-component
scores transformed by $\tanh(t/2)$, with class interactions.
All sixth-class coefficients are fixed at zero, a restriction
beyond removing the common intercept.
The remaining $30$ coefficients lie in $[-0.5,0.5]$,
with ridge coefficient $0.10$.
Five target folds separate fitting from evaluation.
APS uniforms are independent across observations and candidate
classes, shared between methods, and fixed during candidate evaluation.

Crime uses $210$ trees, minimum leaf size four, and square root
feature subsampling.
The training split is divided $80/20$ between fitting the
predictive mean and a shared truncated-normal scale, constrained
to $[0.02,0.5]$.
State and community identifiers are excluded.
Numeric features with at most $5\%$ missingness are selected
using the mean-training split. Imputation, standardization,
and PCA are also fitted on that split.

For $r=(y-\widehat\mu(x))/\widehat\sigma$, response features are
$\tanh(r/2)$ and $\exp\{-(r-c)^2/(2\cdot0.8^2)\}$,
$c\in\{-1,0,1\}$.
Their products with an intercept and two source principal-component
scores transformed by $\tanh(t/2)$ give twelve coefficients
in $[-0.35,0.35]$, with ridge coefficient $0.50$.
Moments use $64$-point Gauss--Legendre quadrature.
Three target folds separate fitting from evaluation.
Length uses trapezoidal integration on $801$ points across
$[0,1]$. Coverage is evaluated at the observed response.
These fixed datasets and partitioning schemes do not exactly
follow the independent population sampling model of the
coverage theorem.

\paragraph{Evaluation summaries.}
Means and SEs summarize replication- or partition-level
coverage and mean set size.
SEs are sample SDs divided by the square root of the count.
Paired percentage length reductions average
$100(1-\ell_r^{\mathrm{comparison}}/\ell_r^{\mathrm{baseline}})$,
where $\ell_r$ is mean total length in replication or partition $r$.
WCP is the baseline for comparisons with WCP-T.
Paired Student's $t$ intervals use $29$ degrees of freedom
for synthetic experiments and seven for real-data comparisons.

\end{document}